\documentclass[10pt,journal,cspaper,compsoc]{IEEEtran}

\usepackage{amsmath,amsthm,amssymb,amsopn,amsfonts,pdfpages,algorithmic,algorithm,bm}
\usepackage{graphics,hyperref,multicol}
\usepackage{graphicx}
\usepackage{fancybox}
\usepackage{tikz}
\usepackage{bbm}
\usepackage[square, numbers, sort]{natbib}
\usepackage{caption}
\usepackage{url}
\usepackage{marginal}
\usepackage{parskip}

\newcommand{\hX}{\widehat X}

\newcommand{\hY}{\widehat Y}
\newcommand{\hxi}{\widehat\xi}
\newcommand{\hH}{\widehat H}

\newcommand{\p}{\mathbb{P}}
\newcommand{\e}{\mathbb{E}}
\newcommand{\argmax}{\operatornamewithlimits{argmax}}

\newtheorem{theorem}{Theorem}
\newtheorem{proposition}[theorem]{Proposition}
\newtheorem{lemma}[theorem]{Lemma}
\newtheorem{corollary}[theorem]{Corollary}
\theoremstyle{definition}
\newtheorem{definition}[theorem]{Definition}
\newtheorem{remark}{Remark}
\date{}
\makeatletter
\def\thm@space@setup{%
  \thm@preskip=\parskip \thm@postskip=0pt
}
\makeatother
\begin{document}
\title{Community Detection and Classification in Hierarchical Stochastic Blockmodels}
\author{Vince Lyzinski, 
        Minh Tang, Avanti Athreya,
         Youngser Park, Carey E. Priebe 
\IEEEcompsocitemizethanks{\IEEEcompsocthanksitem \noindent V.L. is with Johns Hopkins University Human Language Technology Center of Excellence.  M. T., A. A. and C. E. P. are with Johns Hopkins University Department of Applied Mathematics and Statistics.  Y. P. is with Johns Hopkins University Center for Imaging Sciences. .\protect
}
}
\date{\today}
\IEEEcompsoctitleabstractindextext{%
\begin{abstract}
In disciplines as diverse as social
network analysis and neuroscience, many large graphs are believed to
be composed of loosely connected smaller graph primitives, whose
structure is more amenable to analysis
  We propose a robust, scalable, integrated methodology for {\em community detection} and
{\em community comparison} in graphs.
In our procedure, we first embed a graph into an appropriate Euclidean
space to obtain a low-dimensional representation, and then cluster the
vertices into communities.  We next employ nonparametric graph
inference techniques to identify structural similarity among these
communities.  These two steps are then applied recursively on the
communities, allowing us to detect more fine-grained structure.  We
describe a {\em hierarchical stochastic blockmodel}---namely, a
stochastic blockmodel with a natural hierarchical structure---and
establish conditions under which our algorithm yields consistent
estimates of model parameters and {\em motifs}, which we define to be
stochastically similar groups of subgraphs.  Finally, we demonstrate
the effectiveness of our algorithm in both simulated and real
data. Specifically, we address the problem of locating similar
sub-communities in a partially reconstructed {\it Drosophila} 
connectome and in the social network Friendster.
\end{abstract}}
\maketitle

\section{Introduction}
\label{sec:introduction}
\IEEEPARstart{T}he representation of data as graphs, with the vertices as entities
and the edges as relationships between the entities, is now ubiquitous
in many application domains: for example, social networks, in which vertices
represent individual actors or organizations \cite{wasserman};
neuroscience, in which vertices are neurons or brain regions
\cite{sporns_complex}; and document analysis, in which vertices
represent authors or documents \cite{deSolla}. This representation
has proven invaluable in describing and modeling the intrinsic and
complex structure that underlies these data.

In understanding the structure of large, complex graphs, a central
task is that of identifying and classifying local, lower-dimensional
structure, and more specifically, consistently and scalably estimating
subgraphs and subcommunities.  In disciplines as diverse as social
network analysis and neuroscience, many large graphs are believed to
be composed of loosely connected smaller graph primitives, whose
structure is more amenable to analysis.  For example, the
widely-studied social network Friendster\footnote{available from
  \url{http://snap.stanford.edu/data}}, which has approximately 60
million users and 2 billion edges, is believed to consist of over 1
million communities at local-scale.  Insomuch as the communication
structure of these social communities both influences and is influenced by the
function of the social community, we expect there to be repeated
structure across many of these communities (see Section
\ref{sec:data}).  As a second motivating example, the neuroscientific
{\em cortical column conjecture}
\cite{mountcastle1997columnar,marcus2014atoms} posits that the neocortex of
the human brain employs algorithms composed of repeated instances of a
limited set of computing primitives.  By modeling certain portions of
the cortex as a hierarchical random graph, the cortical column
conjecture can be interpreted as a problem of community detection and
classification within a graph.  While the full data needed to test the
cortical column conjecture is not yet available
\cite{takemura2013visual}, it nonetheless motivates our present
approach of theoretically-sound robust hierarchical community
detection and community classification.

Community detection
for graphs is a well-established field of study, and there are many
techniques and methodologies available, such as those based on
maximizing modularity and likelihood \cite{Bickel2009,networks08:_v, Newman2004},
random walks \cite{pons05:_comput, rosvall08:_maps}, and spectral
clustering and partitioning
\cite{McSherry2001,rohe2011spectral,STFP,von2007tutorial,qin2013dcsbm,chaudhuri12:_spect}. While many of these results focus on the consistency of the algorithms---namely, that the proportion of misclassified vertices goes to zero---the key results in this paper give guarantees on the probability of {\em perfect clustering}, in which no vertices at all are misclassified.  As such, they are similar in spirit to the results of \cite{Bickel2009} and represent a considerable improvement of our earlier clustering results from \cite{perfect}. As might be expected, though, the strength of our results depends on the average degree of the graph, which we require to grow at least at order $\sqrt{n} \log^{2} (n)$. We note that weak or partial recovery results are available for much sparser regimes, e.g., when the average degree stays bounded as the number of vertices $n$ increases (see, for example, the work of \cite{mossel:ptrf}). A partial summary of various consistency results and sparsity regimes in which they hold is given in Table~\ref{tab:clustering}. Existing theoretical results on clustering have also been centered primarily on isolating fine-grained community structure in a network.  A major contribution of this work, then, is a formal delineation of {\em hierarchical} structure in a network and a provably consistent algorithm to uncover communities and subgraphs at multiple scales. 

\begin{table*}
\centering
\begin{tabular}{|c|c|c|c|}
\hline Average degree & Method & Notion of recovery & References \\ \hline
$O(1)$ & semidefinite programming, backtracking random walks & weak recovery & \cite{mossel:ptrf,hajek,chaogao} \\ \hline
$\Omega(\log{n})$ & spectral clustering & weak consistency & \cite{rohe2011spectral,rinaldo_2013,STFP} \\ \hline
$\Omega(\log{n})$ & modularity maximization & strong consistency & \cite{Bickel2009} \\ \hline
$\Omega(\sqrt{n} \log^{2} n)$ & spectral clustering & strong consistency & \cite{perfect} \\ \hline
\end{tabular}
\caption{A summary of some existing results on the consistency of recovering the block assignments in stochastic blockmodel graphs with fixed number of blocks. Weak recovery and weak consistency correspond to the notions that, in the limit as $n \rightarrow \infty$, the proportion of correctly classified vertices is non-zero and approaches one in probability, respectively. Strong consistency corresponds to the notion that the number of misclassified vertices is zero in the limit.}
\label{tab:clustering}
\end{table*}

Moreover, existing community detection algorithms have focused 
mostly on uncovering the subgraphs. Recently, however, the characterization and
classification of these subgraphs into stochastically similar motifs
has emerged as an important area of ongoing research.
Network comparison is a nascent field, and comparatively few techniques have
thus far been proposed; see
\cite{pao11:_statis_infer_random_graph,
  rukhin11,koutra13:_deltac,rosvall10:_mappin,asta14:_geomet,tang14:_nonpar,
  tang14:_semipar}. 
In particular, in \cite{tang14:_nonpar}, the authors exhibit a consistent nonparametric test for the equality of two generating distributions for a pair of random graphs. 
The method is based
on first embedding the networks into Euclidean space 
followed by computing $L_2$ distances between the density
estimates of the resulting embeddings.  
This hypothesis test will play a central role in our present methodology; see Section \ref{sec:Background}.


In the present paper, we introduce a robust, scalable methodology for {\it community detection} and
{\em community comparison} in graphs, with particular application to
social networks and connectomics.  Our techniques build upon
previous work in graph embedding, parameter
estimation, and multi-sample hypothesis testing (see \cite{STFP,
  perfect, tang14:_semipar, tang14:_nonpar}).
Our method proceeds as follows.  
First, we generate a low-dimensional
representation of the graph \cite{STFP}, cluster to detect subgraphs of interest \cite{perfect}, and then employ 
the nonparametric inference techniques of \cite{tang14:_nonpar} to identify
heterogeneous subgraph structures. 
The representation of a
network as a collection of points in Euclidean space allows for a single framework which combines the steps  of community detection via an adapted spectral clustering procedure (Algorithm \ref{alg:main0}) with network comparison
via density estimation. 
Indeed, the streamlined clustering algorithm proposed in this paper, Algorithm \ref{alg:main0}, is well-suited to our hierarchical framework, whereas classical $K$-means may be ill-suited to the pathologies of this model. 
As a consequence, we are able to present in this paper a unified inference procedure in which community detection, motif identification, and larger network comparison are all seamlessly integrated. 

We focus here on a {\em hierarchical} version of the classical stochastic block
model \cite{Wang1987,Holland1983}, in which the larger graph is comprised of smaller subgraphs, each themselves approximately stochastic blockmodels. We emphasize that our model and subsequent theory rely heavily on an affinity assumption at each level of the hierarchy, and we expect our model to be a reasonable surrogate for a wide range of real networks, as corroborated by our empirical results. In our approach, we aim to infer finer-grained structure at each level of our hierarchy, in effect performing a ``top-down" decomposition.  (For a different generative hierarchical model, in which successive-level blocks and memberships are the inference taks, see \cite{Peixoto_HSBM}.) We recall that the stochastic blockmodel (SBM) is an independent-edge random graph model that posits that the probability of connection between any two vertices is a function of the 
{\em block memberships} (i.e., community memberships) of the vertices. As such, the stochastic blockmodel is commonly used to model community structure in graphs.   
While we establish performance guarantees for this methodology in the
setting of hierarchical stochastic blockmodels (HSBM), we demonstrate the
wider effectiveness of our algorithm for simultaneous community
detection and classification in the {\em Drosophila} connectome and
the very-large scale social network Friendster, which has
approximately 60 million users and 2 billion edges. 

We organize the paper as follows.  In Section \ref{sec:Background}, we provide the key definitions in our model, specifically for random dot product graphs, SBM graphs, and HSBM graphs. We summarize recent results on networks comparison from \cite{tang14:_nonpar}, which is critical to our main algorithm, Algorithm \ref{alg:main}.  We also present our novel clustering procedure, Algorithm \ref{alg:main0}. 
In Section \ref{sec:HSBM}, we demonstrate how, under mild model assumptions, Algorithm \ref{alg:main} can be applied to asymptotically almost surely perfectly recover the motif structure in a two-level HSBM, see Theorem \ref{thm:1}. 
In Section \ref{sec:multilevel}, we consider a HSBM with multiple levels and discuss the recursive nature of Algorithm \ref{alg:main}.  
We also extend Theorem \ref{thm:1} to the multi-level HSBM and show, under mild model assumptions, Algorithm \ref{alg:main} again asymptotically almost surely perfectly recovers the hierarchical motif structure in a multi-level HSBM.  
In Section \ref{sec:data}, we demonstrate that Algorithm \ref{alg:main} can be effective in uncovering statistically similar subgraph structure in real data: first, in the {\em Drosophila} connectome, in which we uncover two repeated motifs; and second, in the Friendster social network, in which we decompose the massive network into 15 large subgraphs, each with hundreds of thousands to millions of vertices.  We identify motifs among these Friendster subgraphs, and we compare two subgraphs belonging to different motifs.  We further analyze a particular subgraph from a single motif and demonstrate that we can identify structure at the second (lower) level.  In Section \ref{sec:Conclusion}, we conclude by remarking on refinements and extensions of this approach to community detection.

\section{Background}
\label{sec:Background}
We situate our approach in the context of hierarchical stochastic
blockmodel graphs. We first define the stochastic blockmodel as a
special case of the more general random dot product graph model
\cite{nickel2006random}, which is itself a special case of the more
general latent position random graph \cite{Hoff2002}.  We next
describe our canonical {\em hierarchical stochastic blockmodel}, which
is a stochastic blockmodel that is endowed with a natural hierarchical
structure.

{\bf Notation:} In what follows, for a matrix $M\in\mathbb{R}^{n\times
  m},$ we shall use the notation $M(i,:)$ to denote the $i$-th row of
$M$, and $M(:,i)$ to denote the $i$-th column of $M$.  For a symmetric
matrix $M\in\mathbb{R}^{n\times n},$ we shall denote the (ordered)
spectrum of $M$ via $\lambda_1(M)\geq \lambda_2(M)\geq\cdots\geq
\lambda_n(M).$

We begin by defining the {\em random dot product} graph.
\begin{definition}[$d$-dimensional Random Dot Product Graph (RDPG)]\label{def:rdpg}
  Let $F$ be a distribution on a set $\mathcal{X}\subset\mathbb{R}^d$
  such that $\langle x,x'\rangle\in[0,1]$ for all
  $x,x'\in\mathcal{X}.$ We say that $(X,A)\sim\mathrm{RDPG}(F)$ is an
  instance of a random dot product graph (RDPG) if
  $X=[X_1,\dotsc,X_n]^\top$ with
  $X_1,X_2,\ldots,X_n\stackrel{\text{i.i.d.}}{\sim}F$, and
  $A\in\{0,1\}^{n\times n}$ is a symmetric hollow matrix satisfying
\[ \p[A|X] = \prod_{i>j} (X_i^\top X_j)^{A_{ij}}{(1-X_i^\top X_j)}^{1-A_{ij}}.\]
\end{definition}
\begin{remark}
  We note that non-identifiability is an intrinsic property of random
  dot product graphs. Indeed, for any matrix $X$ and any orthogonal
  matrix $W$, the inner product between any rows $i,j$ of $X$ is
  identical to that between the rows $i,j$ of $XW$. Hence, for any
  probability distribution $F$ on $\mathcal{X}$ and unitary operator
  $U$, the adjacency matrices $A \sim \mathrm{RDPG}(F)$ and $B \sim
  \mathrm{RDPG}(F \circ U)$ are identically distributed.
\end{remark}
We denote the second moment matrix for the vectors $X_i$ by $\Delta=\e(X_1X_1^T)$;  we assume that $\Delta$ is rank $d$.

The stochastic blockmodel can be
framed in the context of random dot product graphs as follows.
\begin{definition} We say that an $n$ vertex graph $(X,A)\sim\mathrm{RDPG}(F)$
is a (positive semidefinite) stochastic blockmodel (SBM) with $K$ blocks if 
the distribution $F$ is a mixture of $K$ point masses,
$$F=\sum_{i=1}^K \pi(i) \delta_{\xi_i},$$ 
where $\vec{\pi}\in(0,1)^K$ satisfies $\sum_i \pi(i)=1$, and 
the distinct latent positions are given by $\xi=[\xi_1,\xi_2,\ldots,\xi_K]^\top\in\mathbb{R}^{K\times d}$.  
In this case, we write $G\sim SBM(n,\vec{\pi},\xi\xi^\top),$
and  we refer to $\xi\xi^\top\in\mathbb{R}^{K,K}$ as the \emph{block probability matrix} of $G$. Moreover, any stochastic blockmodel graphs where the block probability matrix $B$ is positive semidefinite can be formulated as a random dot product graphs where the point masses are the rows of $B^{1/2}$. 
\end{definition}

Many real data networks exhibit hierarchical community structure (for
social network examples, see \cite{clauset04:_findin,
  mariadassou10:_uncov, sales-pardo07:_extrac, leskovec10:_kronec,
  park10:_dynam_bayes, ahn10:_link, Peixoto_HSBM}; for biological examples, see
\cite{mountcastle1997columnar,takemura2013visual,marcus2014atoms}).
To incorporate hierarchical structure into the above RDPG and SBM
framework, we first consider SBM graphs endowed with the following
specific hierarchical structure.

\begin{definition}[2-level Hierarchical stochastic blockmodel (HSBM)]
\label{def:HSBM}
We say that $(X,A)\sim \mathrm{RDPG}(F)$ is an instantiation of a
$D$-dimensional 2-level hierarchical stochastic blockmodel with parameters $(n,\vec{\pi},\{\vec{\pi}_i\}_{i=1}^R,\xi)$ if $F$ can be
written as the mixture
$$F=\sum_{i=1}^R \pi(i) F_i,$$ 
where $\vec\pi\in(0,1)^R$ satisfies $\sum_i \pi(i)=1$, and for each
$i\in[R]$, $F_i$ is itself a mixture of point mass distributions
$$F_i=\sum_{j=1}^{K_i} \pi_i(j) \delta_{\xi^{(2)}_i(:,j)}$$ 
where $\vec\pi_i\in(0,1)^{K_i}$ satisfies $\sum_j \pi_i(j)=1$.  
The distinct
latent positions are then given by
$\xi=[\xi^{(2)}_1|\cdots
|\xi^{(2)}_R]^\top\in\mathbb{R}^{(\sum_i K_i)\times D}$.
We then write
$$G\sim\text{HSBM}(n,\vec{\pi},\{\vec{\pi}_i\}_{i=1}^R,\xi\xi^\top).$$
\end{definition}
Simply stated, an HSBM graph is an RDPG for which the
vertex set can be partitioned into $R$ subgraphs---where $(\xi^{(2)}_i)^T\in\mathbb{R}^{ K_i\times D}$ denotes the matrix whose rows are the latent positions
characterizing the block probability matrix for subgraph $i$---each of which
is itself an SBM.

Throughout this manuscript, we will make a number of simplifying assumptions on the underlying HSBM in order to facilitate theoretical developments and ease exposition.\\
{\bf Assumption 1.} (Affinity HSBM) We further assume that for the distinct latent positions, if we define
\begin{align}
q&:=\min_{i,h,l} \langle \xi^{(2)}_i(:,l), \xi^{(2)}_i(:,h) \rangle;\\
p&:=\max_{i\neq j}\max_{\ell,h}\big\langle \xi^{(2)}_i(:,\ell),\xi^{(2)}_j(:,h)\big\rangle
\end{align} 
then  $p<q$.
Simply stated, we require that within each subgraph, the connections are comparatively dense, and between two subgraphs, comparatively sparse.\\
{\bf Assumption 2.}  (Subspace structure) To simplify exposition, and to assure the condition that
$\big\langle \xi^{(2)}_i(:,\ell),\xi^{(2)}_j(:,h)\big\rangle\leq p$ for
$1\leq i\neq j\leq R$ and $\ell,h\in[K],$ we impose additional
structure on the matrix of latent positions in the HSBM.  
To wit, we write
$\xi\in\mathbb{R}^{RK\times D}$ explicitly as
\begin{equation*}
\label{eq:X}
\xi=\begin{bmatrix}
       (\xi^{(2)}_1)^T\\
       (\xi^{(2)}_2)^T\\
       \vdots\\
       (\xi^{(2)}_R)^T
     \end{bmatrix}=\underbrace{\begin{bmatrix}
       \xi_1^T & 0 & \cdots&0\\
       0 & \xi_2^T           & \cdots &0\\
       \vdots           & \vdots& \ddots&\vdots\\
       0&0&\cdots&\xi_R^T
     \end{bmatrix}}_{Z^{(2)}}+\mathcal{A}^{(2)},
\end{equation*}
where $Z^{(2)}\circ \mathcal{A}^{(2)}=0 $ ($\circ$ being the Hadamard product)
and the entries of $\mathcal{A}^{(2)}$ are chosen to make
the off block-diagonal elements of the corresponding edge probability
matrix $\xi\xi^T$ bounded above by the absolute constant $p$.  
Moreover, to ease exposition in this 2-level setting, we will assume that for each $i\in[R]$, $\xi_i\in \mathbb{R}^{K\times d}$ so that $D=Rd$.  In practice, the subspaces pertaining to the individual subgraphs need not be the same rank, and the subgraphs need not have the same number of blocks (see Section \ref{sec:data} for examples of $K$ and $d$ varying across subgraphs).


\begin{remark}
Note that 
$G\sim\text{HSBM}(n,\vec{\pi},\{\vec{\pi}_i\}_{i=1}^R,\xi\xi^\top)$
can be viewed as a SBM graph with
$\sum_i K_i$ blocks; 
$G \sim SBM(n,(\pi(1)\vec{\pi}_1,\pi(2)\vec{\pi}_2,\ldots,\pi(R)\vec{\pi}_R),\xi\xi^\top)$. 
However, in this paper we will consider blockmodels with statistically similar subgraphs across blocks, and in general, such models can be parameterized by far fewer than $\sum_i K_i$ blocks. In contrast, when the graph is viewed as an RDPG, the full $\sum_i K_i$ dimensions may be needed, because our affinity assumption necessitates a growing number of dimensions to accommodate the potentially growing number of subgraphs. Because latent positions associated to vertices in different subgraphs must exhibit near orthogonality, teasing out the maximum possible number of subgraphs for a given embedding dimension is, in essence, a cone-packing problem; while undoubtedly of interest, we do not pursue this problem further in this manuscript.  
\end{remark}



Given a graph from this model satisfying Assumptions 1 and 2, we use Algorithm~\ref{alg:main} to uncover the hidden hierarchical structure.  Furthermore, we note that Algorithm \ref{alg:main} can be applied to uncover hierarchical structure in any hierarchical network, regardless of HSBM model assumptions.
However, our theoretical contributions are proven under HSBM model assumptions.

A key component of
this algorithm is the computation of the adjacency spectral embedding \cite{STFP},
defined as follows.
\begin{definition}
  \label{ase}
 Given an adjacency matrix $A \in \{0,1\}^{n \times n}$ of a $d$-dimensional RDPG($F$),  
  the {\em adjacency spectral embedding} (ASE) of $A$ into
  $\mathbb{R}^{d}$ is given by $\hX=U_A
  S_A^{1/2}$ where
$$|A|=[U_A|\tilde U_A][S_A
\oplus \tilde S_A][U_A|\tilde U_A]$$ is the spectral decomposition of $|A| = (A^{\top} A)^{1/2}$, $S_A$ is the diagonal matrix with the (ordered) $d$ largest
eigenvalues of $|A|$ on its diagonal, and $U_A\in\mathbb{R}^{n\times d}$ is the matrix whose columns are the corresponding orthonormal eigenvectors of $|A|$.
\end{definition} 

\begin{algorithm}[t!]
  \begin{algorithmic}
    \STATE \textbf{Input}: Adjacency matrix $A\in
    \{0,1\}^{n\times n}$ for a latent position random graph.
    \STATE \textbf{Output}: Subgraphs and characterization of their dissimilarity
\WHILE {Cluster size exceeds threshold}
\STATE {\em Step 1}: 
Compute the adjacency spectral embedding (see Definition~\ref{ase}) $\hX$ of $A$ into $\mathbb{R}^{D}$;
   \STATE {\em Step 2}: Cluster $\hX$ to obtain subgraphs $\hH_1,
    \cdots, \hH_R$ using the procedure described in Algorithm~\ref{alg:main0}
    \STATE {\em Step 3}: For each $i\in[R],$ compute the adjacency
    spectral embedding for each subgraph
    $\hH_i$ into $\mathbb{R}^{d}$, obtaining $\hX_{\hH_i}$; 
    \STATE \textit{Step 4}: Compute $\widehat S:=[T_{\hat n_r,\hat n_s}(\hX_{\hH_r}, \hX_{\hH_s})]$, where $T$ is the test statistic in Theorem~\ref{thm:mmd_unbiased_ase}, producing a pairwise dissimilarity matrix on induced subgraphs;
\STATE \textit{Step 5}: Cluster induced subgraphs into motifs using the dissimilarities given in $\widehat S$; e.g., use a hierarchical clustering algorithm to cluster the rows of $\widehat{S}$ or the matrix of associated $p$-values.
\STATE \textit{Step 6}: Recurse on a representative subgraph from each motif (e.g., the largest subgraph), embedding into $\mathbb{R}^d$ in Step 1 (not $\mathbb{R}^D$);
\ENDWHILE
\end{algorithmic}
\caption{Detecting hierarchical structure for graphs}
\label{alg:main}
\end{algorithm} 
It is proved in \cite{STFP,sussman12:_univer} that the adjacency
spectral embedding provides a consistent estimate of the true latent
positions in random dot product graphs. The key to this result is a
tight concentration, in Frobenius norm, of the adjacency spectral
embedding, $\widehat X$, about the true latent positions $X$.  This
bound is strengthened in \cite{perfect}, wherein the authors show
tight concentration, in $2\mapsto\infty$ norm, of $\widehat X$ about
$X$.  The $2 \mapsto \infty$ concentration provides a significant
improvement over results that employ bounds on the Frobenius norm of
the residuals between the estimated and true latent positions, namely
$\|\hX-X\|_F$. The Frobenius norm bounds are potentially sub-optimal
for subsequent inference, because one cannot rule out that a
diminishing but positive proportion of the embedded points contribute
disproportionately to the global error.

However, the $2\mapsto\infty$ norm concentration result in
\cite{perfect} relies on the assumption that the eigenvalues of
$\e[X_1 X_1^T]$ are distinct , which is often violated in the setting
of repeated motifs for an HSBM.  One of the main contributions of this
paper is a further strengthening of the results of \cite{perfect}: in
Theorem \ref{thm:minh}, we prove that $\widehat X$ concentrates about
$X$ in $2\mapsto\infty$ norm with far less restrictive assumptions on
the eigenstructure of $\e[X_1 X_1^T]$.

In this paper, if $E_n$ is a sequence of events, we say that $E_n$ {\em occurs asymptotically almost surely} if $P(E_n) \rightarrow 1$ as $n \rightarrow \infty$; more precisely, we say that $E_n$ occurs asymptotically almost surely if for any fixed $c>0$, there exists $n_0(c)$ such that if $n>n_0(c)$ and $\eta$ satisfies $n^{-c}<\eta<1/2$, then $P(E_n)$ is at least $1-\eta$.  The theorem below asserts that the $2 \mapsto \infty$ norm of the differences between true and estimated latent positions is of a certain order asymptotically almost surely. In the appendix, we state and prove a generalization of this result in the non-dense regime.

\begin{theorem}
\label{thm:minh}
Let $(A, X) \sim \mathrm{RDPG}(F)$ where the second moment matrix $\Delta=\e(X_1 X_1^T)$ is of rank $d$. 
Let $E_n$ be the event that there exists  
a rotation matrix $W$ such that 
\begin{equation*}
 \|\hat{X} - X W\|_{2 \rightarrow \infty}= \max_{i} \| \hat{X}(i,:) -
 W X(i,:) \| \leq \frac{C d^{1/2} \log^2{n}}{\sqrt{n}}
   \end{equation*}
   where $C$ is some fixed constant. Then $E_n$ occurs asymptotically
   almost surely.
\end{theorem}


We stress that because of this bound on the $2\rightarrow \infty$
norm, we have far greater control of the errors in individual rows of
the residuals $\hX-X$ than possible with existing Frobenius norm bounds.
One consequence of this control is that an asymptotically perfect clustering procedure for $X$ will yield an equivalent asymptotically almost surely perfect clustering of $\hX$.
This insight is the key to proving Lemma \ref{lem:perfect}, see the appendix for full detail.
A further consequence of Theorem \ref{thm:minh}, in the setting of
random dot product graphs without a canonical block structure,
is that one can choose a loss function with respect to which ASE followed by a suitable clustering yields optimal
clusters \cite{sussman12:_univer,perfect}.  This implies that
meaningful clustering can be pursued even when no canonical
hierarchical structure exists.

Having successfully embedded the graph $G$ into $\mathbb{R}^D$ through
the adjacency spectral embedding, we next cluster the vertices of $G$,
i.e., rows of $\hX$.  
For each $i\in[R]$, we define 
$$(\widehat \xi^{(2)}_i)^T\in\mathbb{R}^{|V(H_i)|\times D}$$ to be the matrix whose rows are the rows in $\hX$ corresponding to the latent positions in the rows of $(\xi^{(2)}_i)^T$.
Our clustering algorithm proceeds as follows.  
With Assumptions 1 and 2, and further assuming that $R$ is known, we first build a ``seed''
set $\mathcal{S}_n$ 
 as follows.
Initialize $\mathcal{S}_0$ to be a random sampling of $R$ rows of $\hX.$
For each $i\in[n]$, let $\tilde{y},\tilde{z}\in\mathcal{S}_{i-1}$ be such that
$$\max_{y, z\in \mathcal{S}_{i-1}}\langle y, z \rangle=\langle \tilde{y}, \tilde{z} \rangle. $$
If $
\max_{x \in \mathcal{S}_{i-1}} \langle \hX(i,:),x\rangle< \langle \tilde{y},\tilde{z}\rangle,$
then add $\hX(i,:)$ to $\mathcal{S}_{i-1}$, and remove $\tilde{z}$ from $\mathcal{S}_{i-1}$; i.e., 
$$\mathcal{S}_i=(\mathcal{S}_{i-1}\setminus\{\tilde z\})\cup\{\hX(i,:)\}.$$ 
If 
$
\max_{x \in \mathcal{S}_{i-1}} \langle \hX(i,:),x\rangle\geq \langle \tilde{y},\tilde{z}\rangle,$
then set $\mathcal{S}_i=\mathcal{S}_{i-1}.$
Iterate this procedure until all $n$ rows of $\hX$ have been considered.  
We show in Proposition~\ref{prop:casnaaccuracy} in the appendix that
$\mathcal{S}_n$ is composed of exactly one row from each
$\widehat{\xi}^{(i)}$. Given the seed set 
$\mathcal{S}_n=\{s_1, s_2, \cdots, s_R\}$, we then
initialize $R$ clusters $\widehat C_1, \widehat C_2, \cdots, \widehat C_R$ via $\widehat C_i=\{s_i\}$ for each $i\in[R].$ 
Lastly, for $i \in [n]$, assign $\hX(i, :)$ to $\widehat C_j$ if 
$$\text{arg}\max_{s \in \mathcal{S}_n} \langle \hX(i, :), s \rangle =s_j.$$ 
As encapsulated in the next lemma, 
this procedure, summarized in Algorithm \ref{alg:main0}, yields an asymptotically perfect clustering of the rows of $\hX$ for HSBM's under mild model assumptions.

\begin{algorithm}[t!]
\begin{algorithmic}
  \STATE Initialize $\mathcal{S}_0$ to be a random sampling of $R$ rows of $\hat{X}$. 
  \FORALL{ $i \in [n]$}
   \STATE Let $\tilde{y}, \tilde{z} \in \mathcal{S}_{i-1}$ be
   such that $\langle \tilde{y}, \tilde{z} \rangle = \max_{y,z \in \mathcal{S}_{i-1}} \langle y, z \rangle$
   \IF{$\max_{x \in \mathcal{S}_{i-1}} \langle \hat{X}(i,:), x \rangle \leq
     \langle \tilde{y}, \tilde{z} \rangle$} 
   \STATE $\mathcal{S}_i = (\mathcal{S}_{i-1} \setminus \{\tilde{z}\}) \cup \{\hat{X}(i,:)\}$
   \ENDIF
  \ENDFOR
  \STATE Denote $\mathcal{S}_n = \{s_1, \dots, s_R \}$
  \STATE Initialize $R$ clusters $\widehat C_1 = \{s_1\}, \dots, \widehat C_R =
  \{s_R\}$ 
\FORALL{ $i \in [n]$}
  \STATE Let $\hat{\tau}(i) = \argmax_{j \in R} \langle \hX(i,:), s_j \rangle$
  \STATE $\widehat C_{\hat{\tau}(i)} = \widehat C_{\hat{\tau}(i)} \cup \{\hX(i,:) \}$
  \ENDFOR
  \caption{Seeded nearest neighbor subspace clustering}
  \label{alg:main0}
\end{algorithmic}
\end{algorithm}

\begin{lemma}
\label{lem:perfect}
Let $G \sim \mathrm{HSBM}(n, \vec{\pi}, \{\vec{\pi}_i\}_{i=1}^{R}, \xi \xi^{\top})$ satisfying Assumptions 1 and 2, and suppose further that $\pi_{\min}:=\min_i\pi(i)>0$.
Then asymptotically almost surely,
\begin{equation*}
\min_{\sigma \in S_R} \sum_{i} \mathbbm{1} \{ \hat{\tau}(i) \not = \sigma(\tau(i)) \} = 0,
\end{equation*}
where $\tau:[n] \rightarrow [R]$ is the true assignment of vertices to subgraphs, and $\hat{\tau}$ is the assignment given by our clustering procedure above. 
\end{lemma} 


Under only our ``affinity assumption"---namely that $q>p$---$k$-means
cannot provide a provably perfect clustering of vertices. This is a
consequence of the fact that the number of clusters we seek is far
less than the total number of distinct latent positions. As a notional
example, consider a graph with two subgraphs, each of which is an SBM
with two blocks. The representation of such a graph in terms of its
latent positions is illustrated in Figure~\ref{fig:hsbm_cluster}. We
are interested in clustering the vertices into subgraphs, i.e., we
want to assign the points to their corresponding cones (depicted via
the shaded light blue and pink areas). If we denote by $\pi_1$,
$\pi_2$, and $\pi_3$ the fraction of red, green, and blue colored
points, respectively, then a $k$-means clustering of the colored points
into two clusters might, depending on the distance between the points
and $\pi_1, \pi_2, \pi_3$, yield two clusters with cluster centroids
inside the same cone -- thereby assigning vertices from different
subgraphs to the same cluster. That is to say, if the subgraphs' sizes
in Figure~\ref{fig:hsbm_cluster} are sufficiently
unbalanced, then $k$-means clustering could yield a clustering in which
the yellow, green, and blue colored
points are assigned to one cluster, and the red colored points are
assigned to another cluster. 
In short, $k$-means is not a subspace clustering algorithm, and the subspace and affinity assumptions made in our HSBM formulation (Assumptions 1, 2, 3 and 4) render $k$-means suboptimal for uncovering the subgraph structure in our model.
Understanding the structure uncovered by $k$-means in our HSBM setting, while of interest, is beyond the scope of this manuscript.

Note that $p$ being small ensures that the subgraphs of interest,
namely the $H_i$'s, lie in nearly orthogonal subspaces of
$\mathbb{R}^D$.  Our clustering procedure is thus similar in spirit to
the subspace clustering procedure of \cite{vidal2010tutorial}. 

\begin{figure}[tp]
  \centering
  \includegraphics[width=0.4\textwidth]{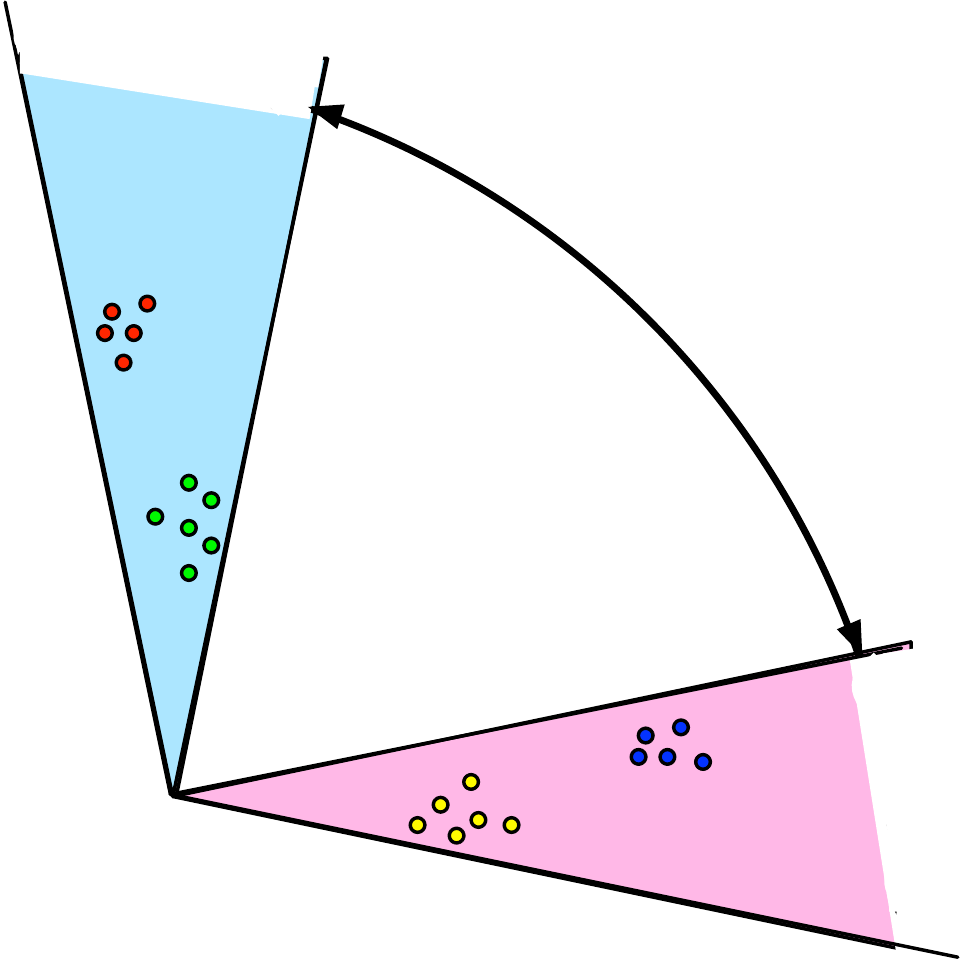}
  \caption{Subgraphs vs. clustering:  Note that if the fraction of points in the pink cone is sufficiently large, $K$-means clustering (with $K=2)$ will not cluster the vertices into the canonical subgraphs. }
  \label{fig:hsbm_cluster}
\end{figure}
\begin{remark}
\label{rem:model}

In what follows, we will assume that $R$, the number of induced SBM
subgraphs in $G$, and $D$ are known a priori.  In practice, however,
we often need to estimate both $D$ (prior to embedding) and $R$ (prior
to clustering).  To estimate $D$, we can use singular value
thresholding \cite{chatterjee2014matrix} to estimate $D$ from a
partial SCREE plot. 
While we can estimate $R$ via traditional
techniques---i.e., measuring the validity of the clustering provided by Algorithm \ref{alg:main0} over
a range of $R$ via silhouette width (see \cite[Chapter 3]{kaufman2009finding})---we propose an alternate estimation procedure tuned to our algorithm.
For each $k=2,3,\ldots,\widehat D$, we run Algorithm \ref{alg:main0} with $R=k$, and repeat this procedure $n_{MC}$ times. 
For each $k$, and each $i=1,2,\ldots,n_{MC}$ compute
$$\phi_i^{(k)}=\max_{s,t\in S_n}\langle s,t\rangle,$$
and compute
$$\phi^{(k)}=\frac{1}{n_{MC}}\sum_{i=1}^{n_{MC}}\phi^{(k)}_i.$$
If the true $R$ is greater than or equal to $k$, then we expect $\phi^{(k)}$ to be small by construction.
If $k$ is bigger than the true $R$, then at least two of the vectors in $S_n$ would lie in the same subspace; i.e., their dot product would be large.
Hence, we would expect the associated $\phi^{(k)}$ to be large.
We employ standard ``elbow-finding'' methodologies \cite{zhu2006automatic} to find the value of $k$ for which $\phi^{(k)}$ goes from small to large, and this $k$ will be our estimate of $R$.
As Algorithm \ref{alg:main0} has running time linear in $n$, with a bounded number of Monte Carlo iterates, this estimation procedure also has running time linear in $n$.
\end{remark}

Post-clustering, a further question of interest is to 
determine which of those induced subgraphs are structurally similar.
We define a motif as a collection of distributionally
``equivalent''---in a sense that we will make precise in Definition \ref{motif}---RDPG graphs. 
An
example of a HSBM graph with $8$ blocks
in $3$ motifs is presented in Figure~\ref{fig:hsbm-example1}.
\begin{figure}[tbp]
  \centering
  \includegraphics[width=0.4\textwidth]{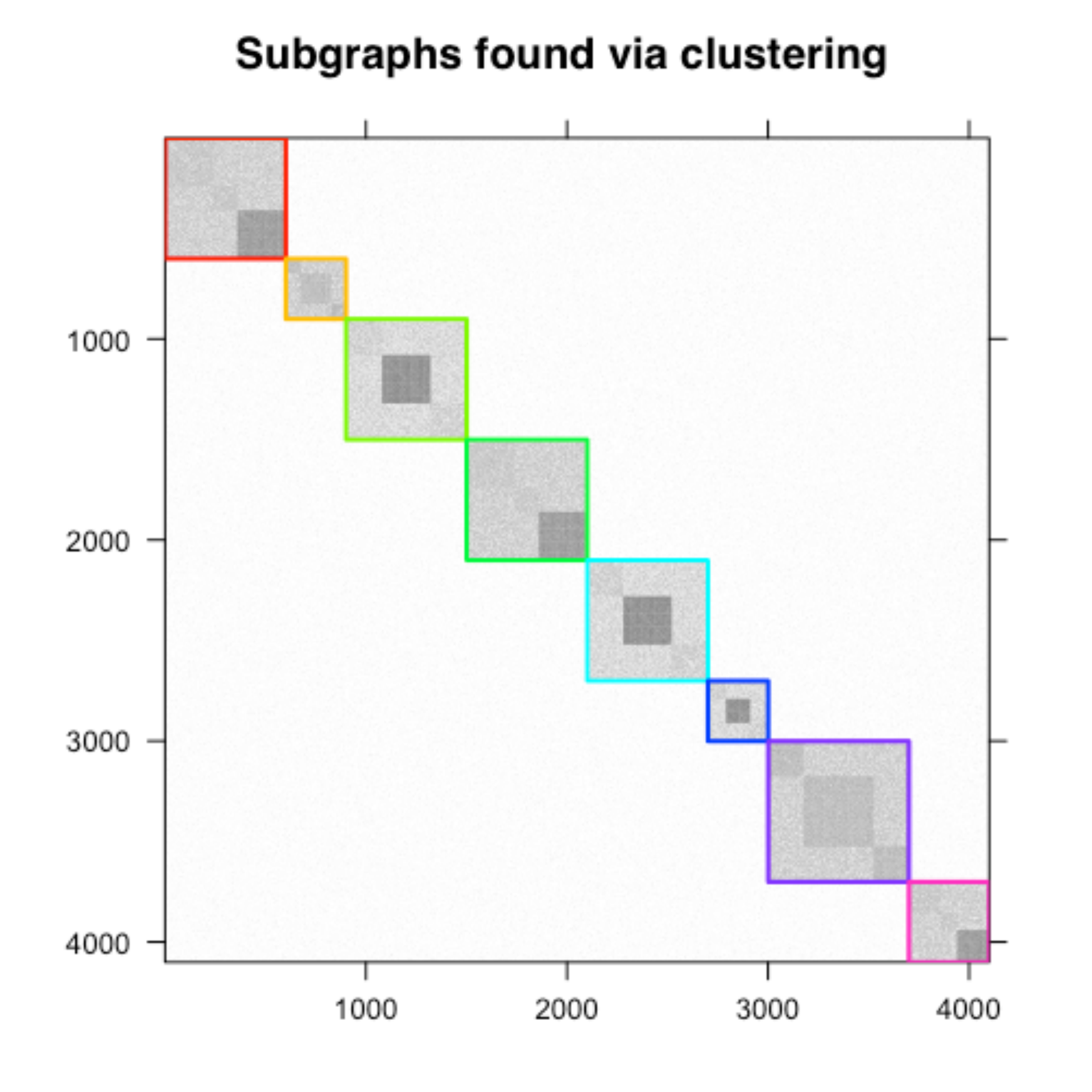}
  \caption{Depiction of the adjacency matrix of a two-level HSBM graph with 3 distinct motifs.  In the above 4100$\times$4100 grid, if an edge exists in $G$ between vertices $i$ and $j$, the the corresponding $i,j$-th cell in the grid is black.  The cell is white if no edge is present.
  The subgraphs corresponding to  motifs are $H_1$, $H_4$, and $H_8$; $H_2$, and $H_7$; and $H_3$, $H_5$, and $H_6$.}
  \label{fig:hsbm-example1}
\end{figure}
More precisely, we define a {\em motif}---namely, an
equivalence class of random graphs---as follows.
\begin{definition}\label{motif}
  Let $(A,X)\sim RDPG(F)$ and $(B,Y)\sim RDPG(G)$.  We say that $A$ and $B$ are of the same {\em
  motif} if there exists a unitary transformation $U$ such
  that $F=G \circ U$.  
\end{definition} 

To detect the presence of motifs among the induced subgraphs $\{\hH_1,\ldots,\hH_R\}$, we adopt the nonparametric test procedure of \cite{tang14:_nonpar}
to determine whether two RDPG graphs have the same
underlying distribution.  The principal result of that work is the following:
\begin{theorem}
  \label{thm:mmd_unbiased_ase}
   Let $(A,X)\sim RDPG(F)$ and $(B,Y)\sim RDPG(G)$ be $d$-dimensional random
   dot product graphs. Consider the hypothesis test
  \begin{align*}
    \phantom{against} \quad H_{0} \colon F = G \circ U \quad 
    \text{against} \quad H_{A} \colon F  \not =  G \circ U.
  \end{align*}
  Denote by $\hX = \{\hX_1, \dots, \hX_n\}$
  and $\hY = \{\hY_1, \dots, \hY_m\}$ the
  adjacency spectral embedding of $A$ and
  $B$, respectively. Define the test statistic $T_{n,m} =
  T_{n,m}(\hX, \hY) $ as follows:
  \begin{align}
    \label{eq:11}
    &T_{n,m}(\hX, \hY) = \frac{1}{n(n-1)}
      \sum_{j \not = i}
      \kappa(\hX_i,
      \hX_j)\notag \\ 
      &\hspace{4mm} -\frac{2}{mn} \sum_{i=1}^{n}
      \sum_{k=1}^{m} \kappa(\hX_i,
      \hY_k) + \frac{1}{m(m-1)} \sum_{l \not = k} \kappa(\hY_k, \hY_l)
  \end{align}
  where $\kappa$ is a radial basis kernel, e.g., $\kappa = \exp(-
  \|\cdot - \cdot\|^2/\sigma^2)$. Suppose that $m, n \rightarrow
  \infty$ and $m/(m+n) \rightarrow \rho \in (0,1)$.  Then under the
  null hypothesis of $F = G \circ U$,
  \begin{equation}
\label{eq:conv_mmdXhat_null}
  |T_{n,m}(\hX, \hY) -
    T_{n,m}(X, Y W)| \overset{\mathrm{a.s.}}{\longrightarrow} 0 
  \end{equation}
  and $|T_{n,m}(X, Y W)|\rightarrow 0$ as $n,m\rightarrow\infty$,
  where $W$ is any orthogonal matrix such that $F = G \circ W$. 
  In addition, under the
  alternative hypothesis of $F\neq G\circ U$, there exists an
  orthogonal matrix $W\in\mathbb{R}^{d\times d}$, depending on $F$ and
  $G$ but independent of $m$ and $n$, such that
\begin{equation}
\label{eq:conv_mmdXhat_alt}
    |T_{n,m}(\hX, \hY) -
    T_{n,m}(X, Y W)| \overset{\mathrm{a.s.}}{\longrightarrow} 0,
  \end{equation}
    and $|T_{n,m}(X, Y W)|\rightarrow c>0$ as $n,m\rightarrow\infty$.
\end{theorem}
Theorem~\ref{thm:mmd_unbiased_ase} allows us to formulate the problem
of detecting when two graphs $A$ and $B$ belong to the same motif as a
hypothesis test. Furthermore, under appropriate conditions on $\kappa$
(conditions satisfied when $\kappa$ is a Gaussian kernel with
bandwidth $\sigma^2$ for fixed $\sigma$), the hypothesis test is
consistent for any two arbitrary but fixed distributions $F$ and $G$,
i.e., $T_{n,m}(X, Y) \rightarrow 0$ as $n, m \rightarrow \infty$ if
and only if $F = G$. 
\begin{figure}[t!]
  \centering
  \includegraphics[width=0.5\textwidth]{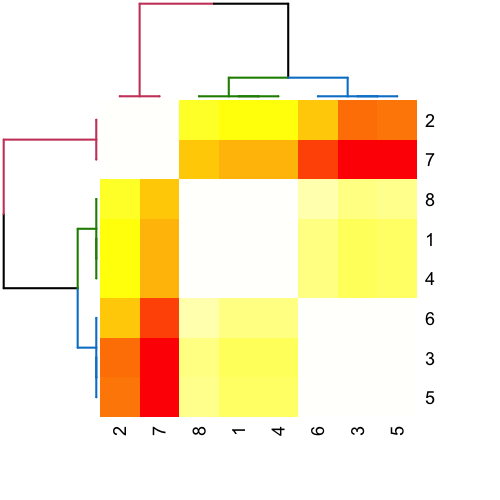}
\caption{Heatmap depicting the dissimilarity matrix $\widehat S$ produced by Algorithm \ref{alg:main} for the 2-level HSBM depicted in Figure \ref{fig:hsbm-example1}. We apply hierarchical clustering to $\widehat S$ (with the resulting dendrogram clustering displayed) demonstrating the which recover the three distinct motifs. }
\label{fig:fignewton}
\end{figure}
We are presently working to extend results on the consistency of
adjacency spectral embedding and two-sample hypothesis testing (i.e.,
Theorem \ref{thm:mmd_unbiased_ase} and \cite{tang14:_semipar}) from
the current setting of random dot product graphs to more general
random graph models, with particular attention to scale-free and
small-world graphs. 
However, the extension of these techniques to more
general random graphs is beset by intrinsic difficulties.  For
example, even extending motif detection to general latent position
random graphs is confounded by the non-identifiability inherent to
graphon estimation. Complicating matters further, there are few random
graph models that are known to admit parsimonious sufficient
statistics suitable for subsequent classical estimation procedures.

\section{Detecting hierarchical structure in the HSBM}
\label{sec:HSBM}
Combining the above inference procedures, our algorithm, as depicted in Algorithm \ref{alg:main}, proceeds as follows.  
We first cluster the adjacency spectral embedding of the graph $G$ to obtain the first-order, large-scale block memberships.  
We then employ the nonparametric test procedure outlined in \cite{tang14:_nonpar} to determine similar induced subgraphs (motifs) associated with these blocks.  
We iterate this process to obtain increasingly refined estimates of the overall graph structure.  
In Step 6 of Algorithm \ref{alg:main}, we recurse on a representative subgraph (e.g., the largest subgraph) within each motif; embedding the subgraph into $\mathbb{R}^d$ (not $\mathbb{R}^D$) as Step 1 of Algorithm \ref{alg:main}.
Ideally, we would leverage the full collection of subgraphs from each motif in this recursion step.
However, the subgraphs within a motif may be of differing orders and meaningfully averaging or aligning them (see \cite{lsgm}) requires novel regularization which, though interesting, is beyond the scope of the present manuscript.

Before presenting our main theorem in the 2-level setting, Theorem \ref{thm:1}, we illustrate the steps of
our method in the analysis of the 2-level synthetic HSBM graph depicted in
Figure \ref{fig:hsbm-example1}.  
The graph has 4100
vertices belonging to 8 different blocks of size $\vec n=(300,600,600,600,700,600,300,400)$ with three distinct motifs. 
The block probability matrices corresponding to these motifs are given by
\begin{equation*}
B_1 = \begin{bmatrix} 0.3 & 0.25 &0.25\\ 0.25 & 0.3 & 0.25 \\ 0.25 & 0.25 & 0.7 \end{bmatrix}; \quad B_2 = \begin{bmatrix} 0.4 & 0.25 & 0.25 \\ 0.25 & 0.4 & 0.25 \\ 0.25 & 0.25 & 0.4
\end{bmatrix}; 
\end{equation*}
\begin{equation*}
\quad B_3 = \begin{bmatrix} 0.25 & 0.2 & 0.2\\ 0.2 & 0.8 & 0.2\\
0.2& 0.2& 0.25 
\end{bmatrix},
\end{equation*}
and the inter-block edge probability is bounded by $p=0.01$.

The
algorithm does indeed detect three motifs, as depicted in Figure \ref{fig:fignewton}.  
The figure presents a heat map depiction of $\widehat S$, and the
similarity of the communities is represented on the spectrum between
white and red, with white representing highly similar communities and
red representing highly dissimilar communities.  
From the figure, we correctly see there are three distinct motif communities, $\{\hH_3,\hH_7\}$, $\{\hH_1,\hH_2,\hH_8\}$, and $\{\hH_4,\hH_5,
\hH_6\}$, corresponding to stochastic
blockmodels with the following block probability matrices
\begin{equation*}
\hat{B}_1 = \begin{bmatrix} 0.27 & 0.25 \\ 0.25 & 0.72 \end{bmatrix}; \quad 
\hat{B}_2 = \begin{bmatrix} 0.41 & 0.27 & 0.26 \\ 0.27 & 0.40 & 0.25 \\ 0.26 & 0.25 & 0.41
\end{bmatrix}.
\end{equation*}
\begin{equation*}
\hat{B}_3 =\begin{bmatrix} 0.22 & 0.20 \\ 0.20 & 0.80 
\end{bmatrix}.
\end{equation*}
\tikzset{
  treenode/.style = {align=center, inner sep=0pt, text centered,
    font=\sffamily},
  arn_x/.style = {treenode, circle, cyan, font=\sffamily\bfseries, draw=cyan,
    text width=3.5em, very thick},
  arn_y/.style = {treenode, circle, orange, font=\sffamily\bfseries, draw=orange,
    text width=3.5em, very thick},
  arn_g/.style = {treenode, circle, green, font=\sffamily\bfseries, draw=green,
    text width=2.5em, very thick},
  arn_r/.style = {treenode, circle, red, draw=red, 
    text width=2.5em, very thick},
  arn_b/.style = {treenode, circle, blue, draw=blue,
    text width=2.5em, very thick}
}
\begin{figure*}[t!]
\centering
\begin{tikzpicture}[level/.style={sibling distance=60mm/#1}]
\node [circle,draw] (z){$G$}
  child {node [circle,draw, arn_y] (a) {$G^{(1)}_1$}
    child {node [circle,draw, arn_g] (b) {$G_{1}^{(2)}$}
      child {node {$\vdots$}
      } 
      child {node {$\vdots$}}
    }
    child {node [circle,draw,arn_b] (g) {$G_{2}^{(2)}$}
    }
  }
  child {node [circle,draw, arn_x] (j) {$G_{2}^{(1)}$}
    child {node [circle,draw, arn_r] (k) {$G_{3}^{(2)}$}
      child {node {$\vdots$}}
      child {node {$\vdots$}}
    }
  child {node [circle,draw, arn_g] (l) {$G_{4}^{(2)}$}
    child {node {$\vdots$}}
    child {node (c){$\vdots$}
    }
  }
  child {node [circle,draw, arn_r] (s) {$G_{5}^{(2)}$}
    child {node {$\vdots$}}
    child {node (c){$\vdots$}
    }
  }
}
  child {node [circle,draw, arn_y] (m) {$G_3^{(1)}$}
    child {node [circle,draw, arn_b] (n) {$G_{6}^{(2)}$}
     }
  child {node [circle,draw, arn_g] (r) {$G_{7}^{(2)}$}
    child {node {$\vdots$}}
    child {node (c){$\vdots$}
     }
    }
  };

\end{tikzpicture}
\caption{Notional depiction of a general hierarchical graph
  structure. The colored nodes in the first and second level of the
  tree (below the root node) correspond to induced subgraphs and associated motifs.}
\end{figure*}
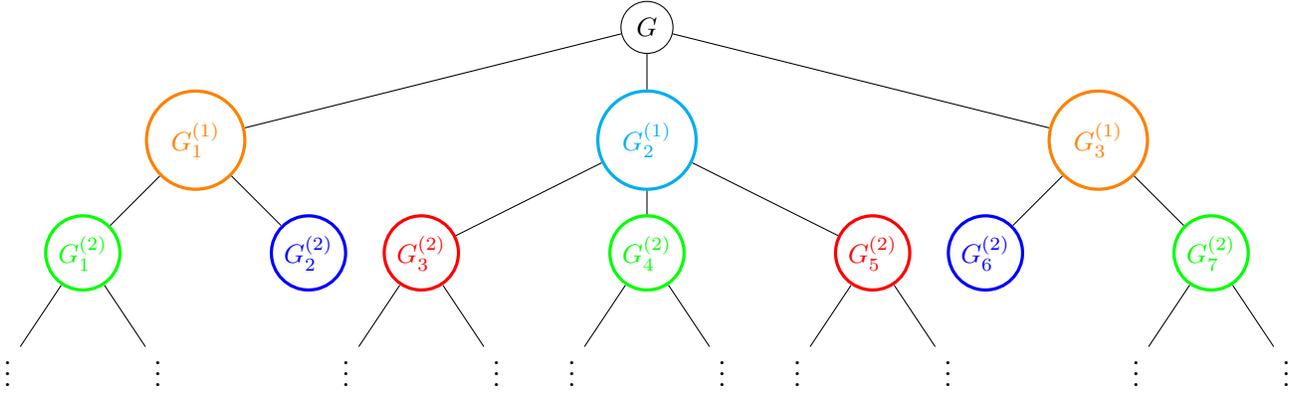

We note that even though the vertices in the HSBM are perfectly clustered into the subgraphs (i.e., for $i\in[8]$, $\hH_i=H_i$ for all $i$), the actual $B$'s differ slightly from their estimates,
but this difference is quite small. 

The performance of Algorithm \ref{alg:main} in this simulation setting
can be seen as a consequence of Theorem \ref{thm:1} below, in which we
prove that under modest assumptions on an underlying 2-level
hierarchical stochastic block model, Algorithm \ref{fig:hsbm-example1}
yields a consistent estimate of the dissimilarity matrix
$S:=[T_{n_i,n_j}(H_i,H_j)].$

\begin{theorem}
\label{thm:1}
Suppose $G$ is a hierarchical stochastic blockmodel satisfying Assumptions 1 and 2. Suppose that
$R$ is fixed and the $\{H_r\}$ correspond to $M$ different motifs,
i.e., the set $\{\xi_1, \xi_2, \dots, \xi_R\}$ has $M \leq R$
distinct elements. Given the assumptions of Theorem \ref{thm:minh} and Lemma \ref{lem:perfect}, 
the procedure in
Algorithm~\ref{alg:main} yields perfect estimates $\hH_1=H_1, \cdots, \hH_R=H_R$ of
$H_1, \cdots, H_R$ and $\widehat{S}$ of $S$ asymptotically almost surely.
\end{theorem}
\begin{proof}
By Lemma \ref{lem:perfect}, the clustering provided by Step 2 of Algorithm \ref{alg:main} will be perfect asymptotically almost surely.
Given this, $\hH_1=H_1, \cdots, \hH_R=H_R$ are consistent estimates of
$H_1, \cdots, H_R$.
Theorem \ref{thm:mmd_unbiased_ase} then implies that $\widehat S$ yields a consistent estimate of $S$; i.e.; for each $i,j$, $|\widehat S(i,j)-S(i,j)|\rightarrow 0$ as $n\rightarrow\infty$.
\end{proof}

With assumptions as in Theorem \ref{thm:1}, any level $\gamma$ test using $S_{ij}$ corresponds to an at most level $\gamma + 2\eta$ test using $\widehat{S}_{ij}$. 
In this case, asymptotically almost surely, the $p$-values of entries of $\widehat S$ corresponding to different motifs will all converge to $0$ as $n\pi_{\min}\rightarrow\infty$, and the $p$-values of entries of $\widehat S$ corresponding to the same motifs will all be bounded away from $0$ as $n\pi_{\min}\rightarrow\infty$.
This immediately leads to the following corollary.
\begin{corollary}
With assumptions as in Theorem \ref{thm:1}, clustering the matrix of $p$-values associated with $\widehat S$ yields a consistent clustering of $\{\hH_i\}_{i=1}^R$ into motifs.
\end{corollary}

Theorem \ref{thm:1} provides a proof of concept inference result for
our algorithm for graphs with simple hierarchical structure, and we
will next extend our setting and theory to a more complex
hierarchical setting.



\section{Multilevel HSBM}
\label{sec:multilevel}

In many real data applications (see for example, Section
\ref{sec:data}), the hierarchical structure of the graph extends
beyond two levels.  We now extend the HSBM model of Definition
\ref{def:HSBM}---which, for ease of exposition, was initially
presented in the 2-level hierarchical setting---to incorporate more
general hierarchical structure.  With the HSBM of Definition
\ref{def:HSBM} being a 2-level HSBM (or {\it 2-HSBM}), we inductively
define an $\ell$-level HSBM (or $\ell${\it -HSBM}) for
$\ell\in\mathbb{Z}\geq 3$ as follows.
\begin{definition}[$\ell$-level Hierarchical stochastic blockmodel $\ell${\it -HSBM}]
We say that $(X,A)\sim\mathrm{RDPG}(F^{(\ell)})$ is an instantiation of a $D^{(\ell)}$-dimensional $\ell$-level HSBM if the distribution $F^{(\ell)}$ can be written as 
\begin{align}
\label{mixture}
F^{(\ell)}=\sum_{i=1}^{R^{(\ell)}}\pi^{(\ell)}(i)F^{(\ell)}_i,
\end{align}
where 
\begin{itemize}
\item[i.] $\vec{\pi}^{(\ell)}\in(0,1)^{R^{(\ell)}}$ with $\sum_i \pi^{(\ell)}(i)=1$; 
\item[ii.]
$F^{(\ell)}$ has support on the rows of 
$[\xi^{(\ell)}_1|\xi^{(\ell)}_2|\cdots|\xi^{(\ell)}_{R^{(\ell)}}]^T,$ where for each $i\in[R^{(\ell)}]$, $F^{(\ell)}_i$ has support on the rows of $(\xi^{(\ell)}_i)^T$.
\item[iii.]
For each $i\in[R^{(\ell)}]$, an RDPG graph drawn according to $(Y,B)\sim$RDPG($F^{(\ell)}_i$) is an $h$-level HSBM with $h\leq\ell-1$ with at least one such $h=\ell-1$.
\end{itemize}  
\end{definition}

Simply stated, an $\ell$-level HSBM graph is an RDPG (in fact, it is an SBM with potentially many more than $R^{(\ell)}$ blocks) for which the
vertex set can be partitioned into $R^{(\ell)}$ subgraphs---where
$(\xi^{(\ell)}_i)^T\in\mathbb{R}^{K_i^{(\ell)}\times D^{(\ell)}}$ denotes the matrix whose rows are the 
latent positions characterizing the block probability matrix for subgraph $i$---each of which is itself an $h$-level HSBM with $h\leq\ell-1$ with at least one such $h=\ell-1$.

As in the 2-level case, to ease notation and facilitate theoretical developments, in this paper we will make the following assumptions on the more general $\ell$-level HSBM.  Letting $(X,A)\sim\mathrm{RDPG}(F^{(\ell)})$ be an instantiation of a $D^{(\ell)}$-dimensional $\ell$-level HSBM, we further assume:\\
{\bf Assumption 3:} (Multilevel affinity) 
For each $k\in\{2,3,\ldots,\ell\}$ the constants
\begin{align}
q^{(k)}&:=\min_{i,j,h} \langle \xi^{(k)}_i(:,j), \xi^{(k)}_i(:,h) \rangle,\\
p^{(k)}&:=\max_{i\neq j}\max_{m,h}\big\langle \xi^{(k)}_i(:,h),\xi^{(k)}_j(:,m)\big\rangle
\end{align}
satisfy $p^{(k)}<q^{(k)}.$\\
{\bf Assumption 4:} (Subspace structure) For each $i\in[R^{(\ell)}]$, $F^{(\ell)}_i$ has support on the rows of $(\xi^{(\ell)}_i)^T$, which collectively satisfy
\begin{align}
\label{eq:ZA}
 \begin{bmatrix}
       (\xi^{(\ell)}_1)^T\\
       (\xi^{(\ell)}_2)^T\\
       \vdots\\
       (\xi^{(\ell)}_{R^{\ell}})^T
     \end{bmatrix}=\underbrace{\begin{bmatrix}
       (\xi^{(\ell-1)}_1)^T & 0 & \cdots&0\\
      0 & (\xi^{(\ell-1)}_2)^T          & \cdots &0\\
       \vdots           & \vdots& \ddots&\vdots\\
       0&0&\cdots&(\xi^{(\ell-1)}_{R^{(\ell)}})^T
     \end{bmatrix}}_{Z^{(\ell)}}+\mathcal{A}^{(\ell)},
     \end{align}
     where $Z^{(\ell)}\circ \mathcal{A}^{(\ell)}=0$ ($\circ$ again being the Hadamard product).
     For each $i\in[R^{(\ell)}]$, an RDPG graph drawn according to $(Y,B)\sim$RDPG($F^{(\ell-1)}_i$) where $F^{(\ell-1)}_i$ has support on the rows of $(\xi^{(\ell-1)}_i\in\mathbb{R}^{K^{(\ell-1)}_i\times D^{(\ell-1)}_i})^T$ and is an at most $\ell-1$-level HSBM (with at least one subgraph being an $(\ell-1)$-level HSBM).
     In addition, we assume similar subspace structure recursively at every level of the hierarchy.

\begin{remark}
\label{rem:lowerdimensional}
As was the case with Assumption 2, Assumption 4 plays a crucial role in our algorithmic development.  
Indeed, under this assumption we can view successive levels of the hierarchical HSBM as RDPG's in successively smaller dimensions.
Indeed, it is these $(Y,B)\sim$RDPG($F^{(\ell-1)}_i$) which we embed in Step 3 of Algorithm \ref{alg:main}, and we embed them into the smaller $\mathbb{R}^{D^{(\ell-1)}_i}$ rather than $\mathbb{R}^{D^{(\ell)}}$.
For example, suppose $G$ is an $\ell$-level HSBM, and $G$ has $R$ subgraphs each of which is an $(\ell-1)$-level HSBM.  Furthermore suppose that each of these subgraphs itself has $R$ subgraphs each of which is an $(\ell-2)$-level HSBM, and so on. 
If the SBM's at the lowest level are all $d$-dimensional, then $G$ can be viewed as an $R^{\ell-1}d$-dimensional RDPG.
In practice, to avoid this curse of dimensionality, we could embed each subgraph at level $k\leq\ell$ into $R$ dimensions and still practically uncover the subgraph (but not the motif!) structure.
This assumption also reinforces the affinity structure of the subgraphs, which is a key component of our theoretical developments.
\end{remark}

In the $2$-level
HSBM setting, we can provide theoretical results on the consistency of
our motif detection procedure, Algorithm \ref{alg:main}.  As it
happens, in this simpler setting, the algorithm terminates after Step
6; that is, after clustering the induced subgraphs into motifs. There
is no further recursion on these motifs.  We next extend Theorem
\ref{thm:1} to the multi-level HSBM setting as follows.  
In the following theorem, for an RDPG $G=(X,A)$,
let $\widehat
X_G$ be the ASE of $G$ and let
$X_G=X$ be the true latent positions of $G$; i.e., $\e(A)=XX^\top.$
\begin{theorem}
\label{thm:multilevel}
With notation as above, let $(X,A)\sim\mathrm{RDPG}(F^{(\ell)})$ be an instantiation of a $D^{(\ell)}$-dimensional, $\ell$-level HSBM with $\ell$ fixed.  Given Assumptions 3 and 4, further suppose that for each $k\in\{2,3,\ldots,\ell\}$ every $k$-level HSBM subgraph, $G$, of $(X,A)$ satisfies
\begin{itemize}
\item[i.] the number of components in the mixture, Eq. \ref{mixture} for $G$, which we shall denote by $R^G$, is known and fixed with respect to $n$, and $\{H^G_r\}_{r=1}^{R^G}$ are these $R^G$ subgraphs of $G$;
additionally, the $R^G$ mixture coefficients $\pi_G(\cdot)$ (associated with Eq. \ref{mixture} for $G$) are all strictly greater than 0;
\item[ii.] if $G$ is $D^G$ dimensional, then these $R^G$ subgraphs---when viewed as the subgraphs corresponding the diagonal blocks of $Z^G$ of Eq. \ref{eq:ZA} to be embedded into $<D^G$ dimensions as in Remark \ref{rem:lowerdimensional}---correspond to $M^G$ different motifs with $M^G$ fixed with respect to $n$.
\end{itemize}
It follows then that for all such $G$, 
the procedure in Algorithm \ref{alg:main} simultaneously yields perfect estimates $\widehat H^G_1=H^G_1$, $\widehat H^G_2=H^G_2,\cdots,$ $\widehat H^G_{R^G}=H^G_{R^G}$ of $\{H^G_r\}_{r=1}^{R^G}$ asymptotically almost surely.
It follows then that for for each such $G$, $\widehat{S}^{G}=[T(\hX_{\widehat H^{G}_i},\hX_{\widehat H^{G}_j})]$ yield consistent estimates of $S^{G}=[T(X_{H^{G}_i},X_{H^{G}_j})]$, which allows for the asymptotically almost surely perfect detection of the $M^G$ motifs.
\end{theorem}
We note here that in Theorem \ref{thm:multilevel}, $\ell$ and the total number of subgraphs at each level of the hierarchy are fixed with respect to $n$.
As $n$ increases, the size of each subgraph at each level is also increasing (linearly in $n$), and therefore any separation between $p^{(k)}$ and $q^{(k)}$ at level $k$ will be sufficient to perfectly separate the subgraphs asymptotically almost surely.
The proof of the above theorem then follows immediately from Theorem \ref{thm:1} and induction on $\ell$, and so is omitted.

Theorem \ref{thm:multilevel} states that, under modest assumptions, Algorithm \ref{alg:main} yields perfect motif detection and classification at every level in the hierarchy. From a technical viewpoint, this theorem relies on a $2 \to \infty$ norm bound on the residuals of $\hat{X}$ about $X$ (see \ref{thm:minh_sparsity}), which is crucial to the perfect recovery of precisely the $R_G$ large-scale subraphs. This bound, in turn, only guarantees this perfect recovery of when the average degree is at least of order $\sqrt{n} \log^2(n)$. We surmise that for subsequent inference tasks that are more robust to the identification of the large-scale subgraphs, results can be established in sparser regimes.

Morever, when applying this procedure to graphs which violate our HSBM model assumptions (for example, when applying the procedure to real data), we encounter error propagation inherent to recursive procedures.
In Algorithm \ref{alg:main}, there are three main sources of error propagation: errorful clusterings; the effect of these errorfully-inferred subgraphs on $\widehat S$; and subsequent clustering and analysis within these errorful subgraphs.  We briefly address these three error sources below.

First, finite-sample clustering is inherently errorful and
misclustered vertices contribute to degradation of power in the motif
detection test statistic.  While we prove the asymptotic consistency
of our clustering procedure in Lemma \ref{lem:perfect}, there are a
plethora of other graph clustering procedures we might employ in the
small-sample setting, including modularity-based methods such as
Louvain \cite{networks08:_v} and \texttt{fastgreedy}
\cite{clauset08:_hierar}, and random walk-based methods such as
\texttt{walktrap} \cite{pons05:_comput}. Understanding the impact that
the particular clustering procedure has on subsequent motif detection
is crucial, as is characterizing the common properties of misclustered
vertices; e.g., in a stochastic block model, are misclustered vertices
overwhelmingly likely to be low-degree?

Second, although testing based on $T$ is asymptotically robust to a
modest number of misclustered vertices, namely $o(\max_i n\pi(i))$ vertices,
the finite-sample robustness of this test statistic remains open.
Lastly, we need to understand the robustness properties of further
clustering these errorfully observed motifs.  In
\cite{priebe2014statistical}, the authors propose a model for
errorfully observed random graphs, and study the subsequent impact of
the graph error on vertex classification.  Adapting their model and
methodology to the framework of spectral clustering will be essential
for understanding the robustness properties of our algorithm, and is
the subject of present research.

\section{Experiments}
\label{sec:data}
We next apply our algorithm to two real data networks:  the {\it Drosophila} connectome from \cite{takemura2013visual} and the Friendster social network.  
\subsection{Motif detection in the {\it Drosophila} Connectome}
 \begin{figure*}[t!]
    \centering
  \includegraphics[width=0.75\textwidth]{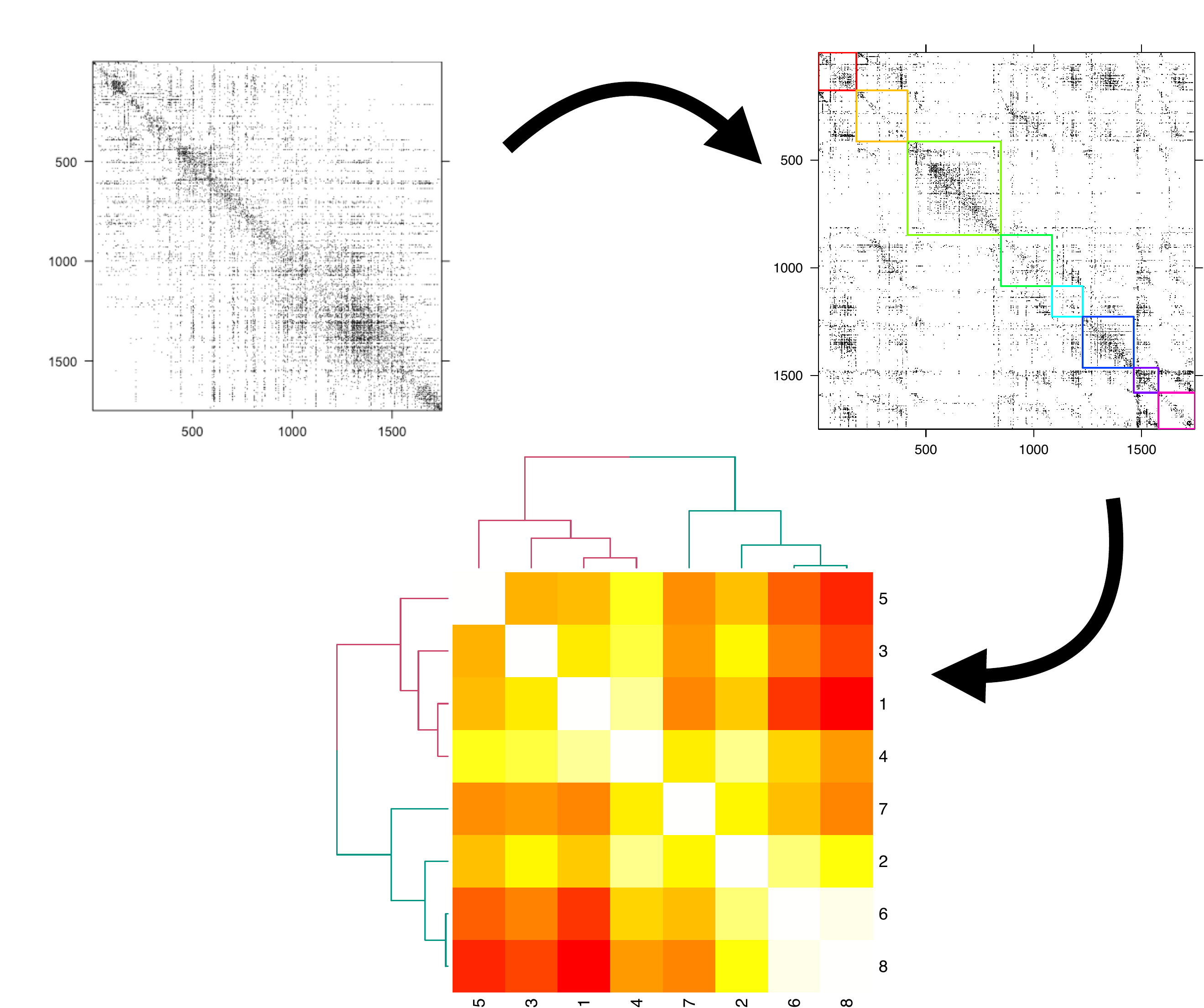}
  \caption{Visualization of our method applied to the {\it Drosophila} connectome.  We show the adjacency matrix (upper left), the clustering derived via ASE, projection to the sphere and clustering via Algorithm \ref{alg:main0}, and lastly $\widehat S$ calculated from these clusters.  
  Clustering the subgraphs based on this $\widehat S$ suggests two repeated motifs: $\{1,4\}$ and $\{2,6,8\}$.  
   Note that the hierarchical clustering also reveals 2nd level motif repetition within the second motif given by $\{6,8\}.$}
  \label{fig:fly1}
  \end{figure*}
The {\it cortical column conjecture} suggests that neurons are connected in a graph which exhibits motifs representing repeated processing modules. 
(Note that we understand that there is controversy surrounding
the definition and even the existence of ``cortical columns''; our consideration includes ``generic'' recurring
circuit motifs, and is not limited to the canonical Mountcastle-style column \cite{mountcastle1997columnar}.)
While the full cortical connectome necessary to rigorously test this conjecture is not yet available even on the scale of fly brains, in \cite{takemura2013visual} the authors were able to construct a portion of the {\it Drosophila} fly medulla connectome which exhibits columnar structure.

This graph is constructed by first constructing the full connectome between 379 named neurons (believed to be a single column) and then sparsely reconstructing the connectome between and within surrounding columns via a semi-automated procedure.
The resulting connectome\footnote{available from the open connectome project \url{http://openconnecto.me/graph-services/download/} (see {\it fly})} has 1748 vertices in its largest connected component, the adjacency matrix of which is visualized in the upper left of Figure \ref{fig:fly1}.  
We visualize our Algorithm \ref{alg:main} run on this graph in Figure \ref{fig:fly1}.
First we embed the graph into $\mathbb{R}^{13}$ ($13$ chosen according the the singular value thresholding method applied to a partial SCREE plot; see Remark \ref{rem:model}) and, to alleviate sparsity concerns, project the embedding onto the sphere.
The resulting points are then clustered into $\widehat R=8$ clusters ($\widehat R$ chosen as in Remark \ref{rem:model}) of sizes $|V(\widehat H_1)|=176,\,|V(\widehat H_2)|= 237,$ $\,|V(\widehat H_3)|=434$, $\,|V(\widehat H_4)|=237,$ $\,|V(\widehat H_5)|=142,$ $\,|V(\widehat H_6)|=237,$ $\,|V(\widehat H_7)|=115$, and $\,|V(\widehat H_8)|=170$ vertices.
These clusters are displayed in the upper right of Figure \ref{fig:fly1}.
We then compute the corresponding $\widehat S$ matrix after re-embedding each of these clusters (bottom of Figure \ref{fig:fly1}). 
In the heat map representation of $\widehat S$, 
the
similarity of the communities is represented on the spectrum between
white and red, with white representing highly similar communities and
red representing highly dissimilar communities. For example, the
bootstrapped $p$-value (from $200$ bootstrap samples) associated with $T(\hat{H}_6, \hat{H}_8)$ is $0.195$, with
$T(\hat{H}_2, \hat{H}_6)$ is $0.02$ and with $T(\hat{H}_6, \hat{H}_1)$
is $0.005$. 


We next apply hierarchical clustering to $\widehat S$ to uncover the repeated motif structure (with the resulting dendrogram displayed in Figure \ref{fig:fly1}).
Both methods uncovered two repeated motifs, the first consisting of
subgraphs 1 and 4 and the second consisting of subgraphs 2, 6, and 8.
Note that the hierarchical clustering also reveals 2nd level motif
repetition within the second motif given by $\{6,8\}.$ Indeed, our method
uncovers repeated {\em hierarchical} structure in this connectome,
and we are presently working with neurobiologists to determine the
biological significance of our clusters and motifs.


\subsection{Motif detection in the Friendster network}

We next apply our methodology to analyze and classify
communities in the Friendster social network. The Friendster social network
contains roughly $60$ million users and $2$ billion
connections/edges.
In addition, there are roughly $1$ million communities at the local scale.
Because we expect the social interactions in these communities to inform the function of the different communities, we expect to observe distributional repetition among the graphs associated with these communities.

Implementing Algorithm \ref{alg:main} on the very large Friendster
graph presents computational challenges.  To overcome this challenge
in scalability, we use the specialized SSD-based graph processing
engine \texttt{FlashGraph} \cite{zheng2014flashgraph}, which is
designed to analyze graphs with billions of nodes.  With
\texttt{FlashGraph}, we adjacency spectral embed the Friendster
adjacency matrix into $\mathbb{R}^{14}$---where $\widehat D=14$ is
chosen using singular value thresholding on the partial SCREE plot
(see Remark \ref{rem:model}).  Using the model selection methodology outlined in Remark \ref{rem:model}, we
find the best coarse-grained clustering of the graph is achieved with
$\widehat R=15$ large-scale clusters ranging
in size from $10^6$ to 15.6 million vertices (note that to alleviate sparsity concerns, we projected the embedding onto the sphere before clustering).  After re-embedding
the induced subgraphs associated with these $15$ clusters, we use a
linear time estimate of the test statistic $T$ to compute $\widehat
S$, the matrix of estimated pairwise dissimilarities among the
subgraphs. 
See Figure
\ref{fig:friend} for a heat map depicting
$\widehat{S}\in\mathbb{R}^{15\times 15}$.  In the heat map, the
similarity of the communities is represented on the spectrum between
white and red, with white representing highly similar communities and
red representing highly dissimilar communities.
From the figure, we can see clear repetition in the subgraph distributions; for example, we see a  repeated motif including subgraphs $\{\widehat H_5, \widehat H_4,\widehat H_3,\widehat H_2\}$ and a clear motif including subgraphs $\{\widehat H_{10},\widehat H_{12},\widehat H_9\}$.

Formalizing the motif detection step, we next employ hierarchical
clustering to cluster $\widehat S$ into motifs; see Figure
\ref{fig:friend} for the corresponding hierarchical clustering
dendrogram, which suggests that our algorithm does in fact uncover
repeated motif structure at the coarse-grained level in the Friendster
graph.  While it may be difficult to draw meaningful inference from
repeated motifs at the scale of hundreds of thousands to millions of
vertices, if these motifs are capturing a common HSBM structure within
the subgraphs in the motif, then we can employ our algorithm
recursively on each motif to tease out further hierarchical structure.
\begin{figure}[t!]
  \centering
  \includegraphics[width=0.4\textwidth]{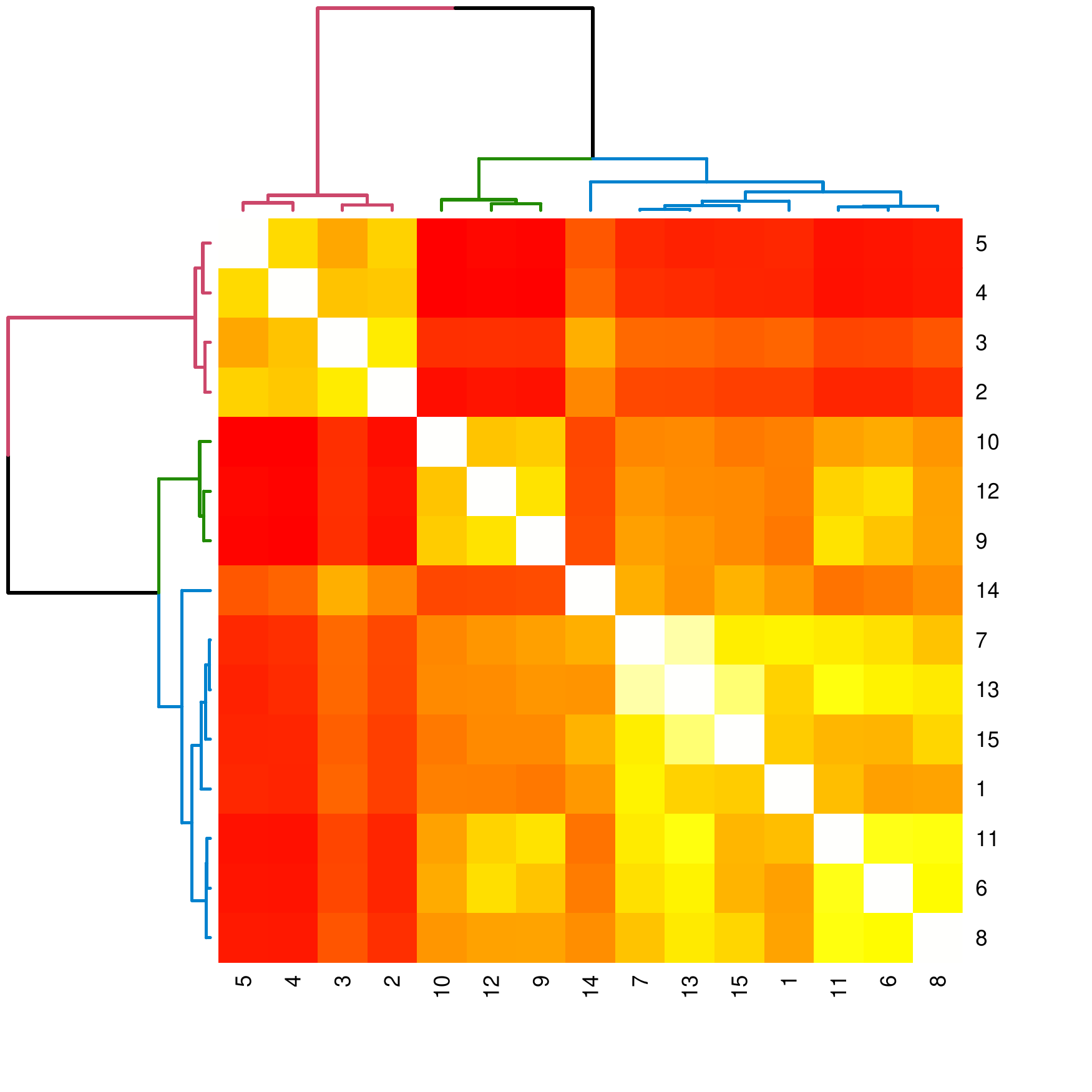}
\caption{Heat map depiction of the level one Friendster estimated dissimilarity matrix $\widehat{S}\in\mathbb{R}^{15\times 15}$.  In the heat map, the similarity of the communities is represented on the spectrum between white and red, with white representing highly similar communities and red representing highly dissimilar communities.  
  In addition, we cluster $\widehat S$ using hierarchical clustering and display the associated hierarchical clustering dendrogram.}
  \label{fig:friend}
\end{figure}

Exploring this further, we consider three subgraphs $\{\widehat
H_2,\widehat H_{8},\widehat H_{15}\}$, two of which are in the same
motif (8 and 15) and both differing significantly from subgraph 2
according to $\widehat S.$ 
We embed these subgraphs into
$\mathbb{R}^{26}$ (26 chosen as outlined in Remark \ref{rem:model}),
perform a Procrustes alignment of the vertex sets of the three
subgraphs, cluster each into $4$ clusters ($4$ chosen
to optimize silhouette width in $k$-means clustering), and estimate both the block connection probability matrices, 
$$\hat P_2=\begin{bmatrix}
0.000045& 0.00080& 0.00056& 0.00047\\
0.00080& 0.025& 0.0096& 0.0072\\
0.00057& 0.0096& 0.012& 0.0067\\
0.00047& 0.0072& 0.0067& 0.023
\end{bmatrix},$$
$$\hat P_8=\begin{bmatrix}
0.0000022& 0.000031& 0.000071& 0.000087\\
0.000031& 0.0097& 0.00046& 0.00020\\
0.000071& 0.00046& 0.0072& 0.0030\\
0.000087& 0.00020& 0.003& 0.016
\end{bmatrix},$$
$$\hat
P_{15}=\begin{bmatrix}
0.0000055& 0.00011& 0.000081& 0.000074\\
0.00011& 0.014& 0.0016& 0.00031\\
0.000081& 0.0016 &0.0065& 0.0022\\
0.000074& 0.00031& 0.0022& 0.019
\end{bmatrix},$$
 and the block membership probabilities $\hat \pi_2,\,\hat \pi_8,\,\hat
\pi_{15},$ for each of the three graphs.  We calculate
\begin{align*}
\|\hat P_2-\hat P_8\|_F&=0.033; \\
\|\hat P_2-\hat P_{15}\|_F&= 0.027;\\
\|\hat P_8-\hat P_{15}\|_F&=0.0058;\\
\|\hat \pi_2-\hat \pi_8\|&=0.043; \\
\|\hat \pi_2-\hat \pi_{15}\|&= 0.043;\\
\|\hat \pi_8-\hat \pi_{15}\|&=0.0010;
\end{align*}
which suggests that the repeated structure our algorithm
uncovers is {\it SBM substructure}, thus ensuring that we can proceed to
apply our algorithm recursively to the subsequent motifs.

As a final point, we recursively apply Algorithm \ref{alg:main} to the
subgraph $\widehat H_{11}$ .  We first embed the graph into
$\mathbb{R}^{26}$ (again, with $26$ chosen as outlined in Remark
\ref{rem:model}).  Next, using the model selection methodology
outlined in Remark \ref{rem:model}, we cluster the vertices into
$\widehat R=13$ large-scale clusters of sizes ranging from 500K to
2.7M vertices.  We then use a linear time estimate of the test
statistic $T$ to compute $\widehat S$ (see Figure \ref{fig:friend2}),
and note that there appear to be clear repeated motifs (for example,
subgraphs 8 and 12) among the $\widehat H$'s.  We run hierarchical
clustering to cluster the $13$ subgraphs, and note that the associated
dendrogram---as shown in Figure \ref{fig:friend2}---shows that our
algorithm again uncovered some repeated level-$2$ structure in the
Friendster network.  We can, of course, recursively apply our
algorithm still further to tease out the motif structure at
increasingly fine-grained scale.

Ideally, when recursively running Algorithm \ref{alg:main}, we would
like to simultaneously embed and cluster all subgraphs in the motif.
In addition to potentially reducing embedding variance, 
being able to efficiently simultaneously embed all the subgraphs in a
motif could greatly increase algorithmic scalability in large networks
with a very large number of communities at local-scale.  In order to
do this, we need to understand the nature of the repeated structure
within the motifs.  This repeated structure can inform an estimation
of a motif average (an averaging of the subgraphs within the motif),
which can then be embedded into an appropriate Euclidean space in lieu
of embedding all of the subgraphs in the motif separately.  However,
this averaging presents several novel challenges, as these subgraphs
may be of very different orders and may be errorfully obtained, which
could lead to compounded errors in the averaging step.  We are
presently working to determine a robust averaging procedure (or a
simultaneous embedding procedure akin to JOFC
\cite{priebe2013manifold}) which exploits the common structure within
the motifs.

 \begin{figure*}[t!]
  \centering
  \includegraphics[width=0.65\textwidth]{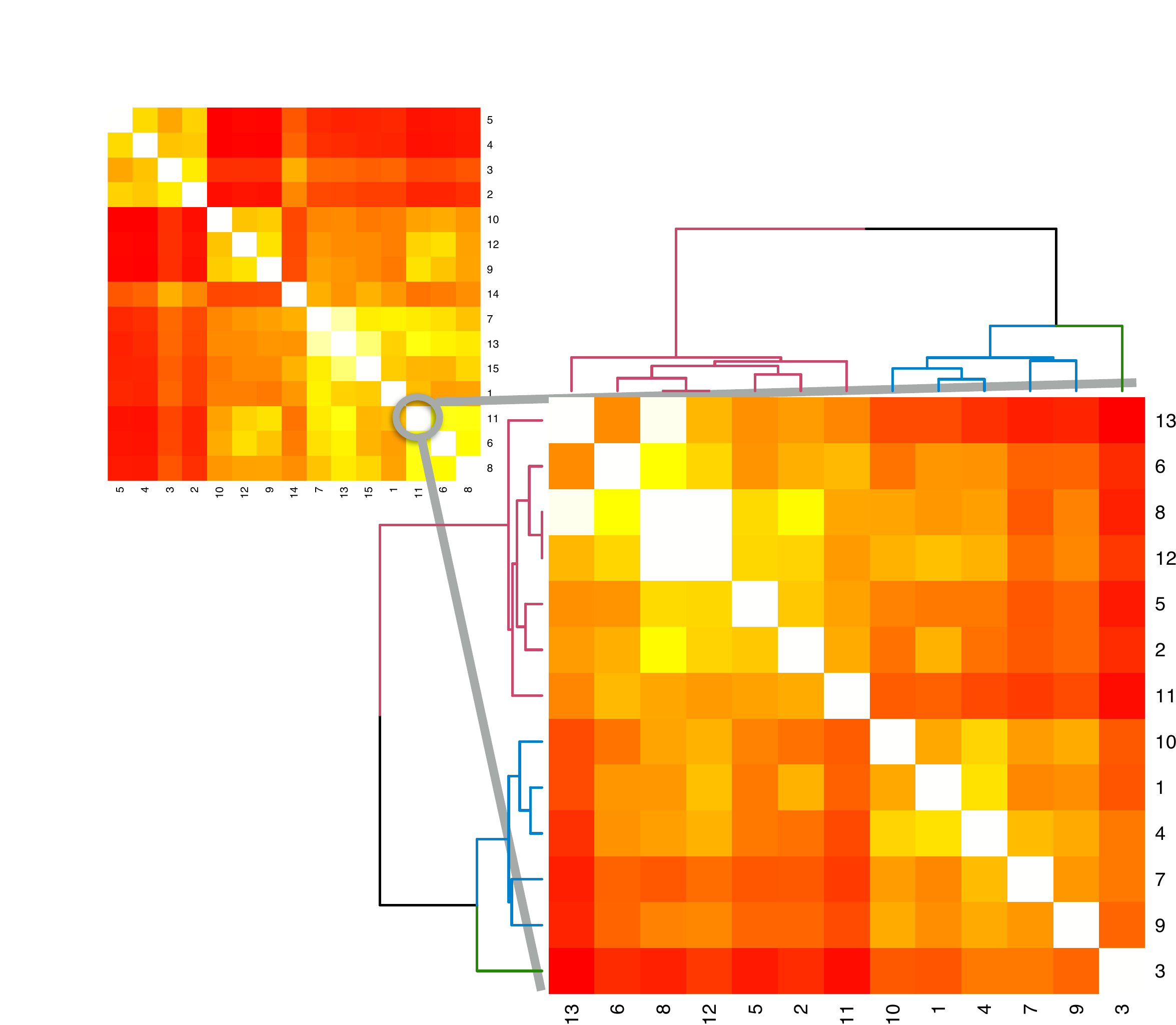}
  \caption{Heat map depiction of the level two Friendster estimated dissimilarity matrix $\widehat{S}\in\mathbb{R}^{13\times 13}$ of $\widehat H_{11}$.  In the heat map, the similarity of the communities is represented on the spectrum between white and red, with white representing highly similar communities and red representing highly dissimilar communities.  
  In addition, we cluster $\widehat S$ using hierarchical clustering and display the associated hierarchical clustering dendrogram.}
  \label{fig:friend2}
\end{figure*}

\section{Conclusion}
\label{sec:Conclusion}
In summary, we provide an algorithm for community detection and classification for hierarchical stochastic blockmodels. Our algorithm depends on a consistent lower-dimensional embedding of the graph, followed by a valid and asymptotically powerful nonparametric test procedure for the determination of distributionally equivalent subgraphs known as motifs.  In the case of a two-level hierarchical stochastic block model, we establish theoretical guarantees on the consistency of our estimates for the induced subgraphs and the validity of our subsequent tests. 

While the hierarchical stochastic block model is a very particular
random graph model, the hierarchical nature of the HSBM---that of smaller subgraphs that are densely connected within and somewhat loosely connected across---is a central feature of many networks.  Because our results are situated primarily in the context of random dot product graphs, and because random dot product graphs can be used to closely approximate many independent edge graphs \cite{tang:_univer}, we believe that our algorithm can be successfully adapted for the determination of multiscale structure in significantly more intricate models.


By performing community detection and classification on the {\em Drosophila} connectome and on the social network Friendster, we demonstrate that our algorithm can be feasibly deployed on real (and, in the case of Friendster, large!) graphs. We leverage state-of-the-art software packages \texttt{FlashGraph} and \texttt{igraph} to substantially reduce computation time. In both graphs, our algorithm detects and classifies multiple similar communities. Of considerable interest and ongoing research is the analysis of the functional or structural features of these distinct communities. Because our algorithm can be applied recursively to uncover finer-grained structure, we are hopeful that these methods can contribute to a deeper understanding of the implications of statistical subgraph similarity on the structure and function of social and biological networks.

\section{Acknowledgments}
The authors thank
 Zheng Da, Disa Mhembere, and Randal Burns for assistance in
 processing the Friendster graph using \texttt{FlashGraph}
 \cite{zheng2014flashgraph}, 
 Gabor Csardi and Edo Airoldi for assistance in implementing our algorithm
 in \texttt{igraph} \cite{csardi2006igraph},
 Daniel L. Sussman for helpful discussions in formulating the HSBM framework, and
 R. Jacob Vogelstein and Joshua T. Vogelstein for suggesting this line of research. 
This work is partially supported by 
 the XDATA \& GRAPHS programs of the Defense Advanced Research Projects Agency (DARPA).
\bibliography{hsbm2}
\bibliographystyle{IEEEtran}
\cleardoublepage
\appendix
\section{Proof of Theorem~\ref{thm:1}}
\label{sec:appendix}
We now provide proofs of  Theorem~\ref{thm:minh} and Lemma~\ref{lem:perfect}. We will state and prove Theorem~\ref{thm:minh} in slightly greater generality here, by first introducing the notion of a random dot product graph with a {\em given sparsity factor} $\rho_n$.
\begin{definition}[The $d$-dimensional random dot product graph with sparsity factor $\rho_n$]
\label{def:rdpg_sparsity}
  Let $F$ be a distribution on a set $\mathcal{X}\subset \mathbb{R}^d$
  satisfying $x^{\top} y \in[0,1]$ for all $x,y\in
  \mathcal{X}$. We say $(X,A)\sim \mathrm{RDPG}(F)$
  with sparsity factor $\rho_n \leq 1$ if the following hold. Let
  $X_1,\dotsc, X_n {\sim} F$ be independent random variables and define
\begin{equation}
\label{eq:defXP}
X=[X_1 \mid \cdots \mid X_n]^\top\in \mathbb{R}^{n\times d}\text{
  and } P= \rho_n X X^\top\in [0,1]^{n\times n}.
\end{equation}
As before, the $X_i$ are the latent positions for the random graph.  The matrix
$A\in\{0,1\}^{n\times n}$ is defined to be a symmetric, hollow matrix such that for all $i<j$, conditioned
on $X_i,X_j$ the $A_{ij}$ are independent and
\begin{equation}
 A_{ij} \sim\mathrm{Bernoulli}(\rho_n X_i^\top X_j),
\end{equation}
namely,
\begin{equation}
\Pr[A \mid X]=\prod_{i <j} (\rho_n X^{\top}_i
X_j)^{A_{ij}}(1- \rho_n X^{\top}_i X_j)^{(1-A_{ij})}
\end{equation}
\end{definition}
Recall that we denote the second moment matrix for the $X_i$ by
$\Delta =\mathrm{E}(X_1X_1^\top)$, and we 
assume that $\Delta$ is of rank $d$. 

\begin{definition}[Embedding of $A$ and $P$]
  Suppose that $A$ is as in Definition~\ref{def:rdpg_sparsity}. Then
  our estimate for the $\rho_n^{1/2} X$ (up to rotation) is
  $\hat{X}= U_{A} S_{A}^{1/2}$, where
  $S_{A} \in \mathbb{R}^{d\times d}$ is the diagonal submatrix with
  the $d$ largest eigenvalues (in magnitude) of $|A|$ and
  $U_{A} \in \mathbb{R}^{n\times d }$ is the matrix whose orthonormal
  columns are the corresponding eigenvectors. \label{def:emb}
  Similarly, we let $U_{P} S_{P} U_{P}^{\top}$ denote the spectral
  decomposition of $P$. Note that $P$ is of rank $d$.
\end{definition}

Theorem~\ref{thm:minh} follows as an easy consequence of the more general Theorem~\ref{thm:minh_sparsity}, which we state below.

\begin{theorem}\label{thm:minh_sparsity}
Let $(A, X) \sim \mathrm{RDPG}(F)$ with rank $d$ second moment matrix and sparsity factor $\rho_n$.  Let $E_n$ be the event that there
exists a rotation matrix $W \in \mathbb{R}^{d \times d}$ such that
\begin{equation*}
  \max_{i} \| \hat{X}_i - \rho_n^{1/2} W X_i \| \leq 
 \frac{C d^{1/2} \log^2{n}}{\sqrt{n\rho_n}}
   \end{equation*}
   where $C > 0$ is some fixed constant. Then $E_n$ occurs asymptotically almost surely.
\end{theorem}

\subsection*{Proof of Theorem~\ref{thm:minh_sparsity}}
\label{sec:minh}
The proof of Theorem~\ref{thm:minh_sparsity} will follow from a succession of supporting results.  
We note that Theorem~\ref{thm:minh_frob}, which deals with the accuracy of spectral embedding estimates in Frobenius norm, may be of independent interest.  In what follows, for a matrix $A\in\mathbb{R}^{m\times m}$, $\|A\|$ will denote the spectral norm of $A$.

We begin with a short proposition.
\begin{proposition}
  \label{prop:uptua_close_rotation}
  Let $(A, X) \sim \mathrm{RDPG}(F)$ with sparsity factor $\rho_n$. Let
  $W_1 \bm{\Sigma} W_2^{\top}$ be the singular value
  decomposition of $U_{P}^{\top}
  U_{A}$. 
  Then asymptotically almost surely,
  \begin{equation*}
    \|U_{P}^{\top} U_{A} -
    W_1 W_2^{\top} \|_{F} = O((n \rho_n)^{-1})
  \end{equation*}
\end{proposition}
\begin{proof}
Let $\sigma_1, \sigma_2, \dots, \sigma_d$ denote the singular values of
  $U_{P}^{\top} U_{A}$ (the diagonal
  entries of $\bm{\Sigma}$). Then $\sigma_i = \cos(\theta_i)$ where
  the $\theta_i$ are the principal angles between the subspaces
  spanned by $U_{A}$ and
  $U_{P}$. Furthermore, by the Davis-Kahan
  $\sin(\Theta)$ theorem (see e.g., Theorem 3.6 in \cite{stewart90:_matrix}),
  \begin{align*}
    \| U_{A} U_{A}^{\top} -
    U_{P} U_{P}^{\top} \| &=
    \max_{i} | \sin(\theta_i) |\leq
    \frac{\|A - P\|}{\lambda_{d}(P)}\\
  \end{align*}
  for sufficiently large $n$. Recall here $\lambda_{d}(P)$ denotes 
  the $d$-th largest eigenvalue of $P$. The spectral norm bound for $A
  - P$ from Theorem 6 in \cite{lu13:_spect}
  then gives 
  $$\| U_{A} U_{A}^{\top} -
    U_{P} U_{P}^{\top} \|\leq\frac{C \sqrt{n \rho_n}}{n \rho_n} = O((n \rho_n)^{-1/2}).$$ 

  We thus have
  \begin{equation*}
    \begin{split}
    \|U_{P}^{\top} U_{A} -
    W_1 W_2^{\top} \|_{F} &= \|\bm{\Sigma} -
    I \|_{F} = \sqrt{\sum_{i=1}^{d} (1 - \sigma_i)^2} \\ &\leq
    \sum_{i=1}^{d} (1 - \sigma_i) \leq \sum_{i=1}^{d} (1 -
    \sigma_i^{2})  \\ & = \sum_{i=1}^{d} \sin^{2}(\theta_i) \\
&\leq d 
    \|U_{A} U_{A}^{T} -
    U_{P} U_{P}^{\top} \|^{2}
   = O((n \rho_n)^{-1})
    \end{split}
  \end{equation*}
  as desired. 
\end{proof}
Denote by $W^{*}$ the orthogonal matrix
$W_1 W_2^{\top}$ as defined in the above
proposition. We now establish the following key lemma. The lemma
allows us to exchange the order of the orthogonal transformation
$W^{*}$ and the diagonal scaling transformation $S_A$ or $S_P$. 
\begin{lemma}
  \label{lem:order_bounds_on_minh_differences}
     Let $(A, X) \sim \mathrm{RDPG}(F)$ with sparsity factor $\rho_n$. Then asymptotically almost surely,
  \begin{equation*}
    \|W^{*} S_{A} - S_{P}
    W^{*} \|_{F} = O(\log n)
\end{equation*}
 and
\begin{equation*}
    \|W^{*} S_{A}^{1/2}  - S_{P}^{1/2}
    W^{*} \|_{F} = O(\log n(n \rho_n)^{-1/2})
  \end{equation*} 
\end{lemma}
\begin{proof}
  Let $R = U_{A} - U_{P}
  U_{P}^{\top} U_{A}$. We note
  that $R$ is the residual after projecting
  $U_{A}$ orthogonally onto the column space of
  $U_{P}$, and note 
\begin{align*} 
\|U_{A} - U_{P}
    U_{P}^{\top} U_{A} \|_{F} = O((n \rho_n)^{-1/2}).
\end{align*}  
  We derive that
  \begin{align*}
    W^{*}& S_{A} =  (W^{*} -
    U_{P}^{\top} U_{A})
    S_{A}  +
    U_{P}^{\top} U_{A}
    S_{A}\\
 &= (W^{*} -
    U_{P}^{\top} U_{A}) S_{A} +
    U_{P}^{\top} A U_{A}
    \\ &=  (W^{*} -
    U_{P}^{\top} U_{A}) S_{A}  +U_{P}^{\top}(A - P)
    U_{A} +
    U_{P}^{\top}P U_{A} 
    \\ &=  (W^{*} -
    U_{P}^{\top} U_{A}) S_{A} +
    U_{P}^{\top}(A - P) R\\
     &\hspace{5mm}+ U_{P}^{\top}(A - P)
    U_{P} U_{P}^{\top}
    U_{A} + U_{P}^{\top}
    P U_{A} \\
    &= (W^{*} -
    U_{P}^{\top} U_{A}) S_{A} +
    U_{P}^{\top}(A - P) R\\
    &\hspace{5mm}+ U_{P}^{\top}(A - P)
    U_{P} U_{P}^{\top}
    U_{A} + S_{P} U_{P}^{\top}
    U_{A} 
  \end{align*}
  Writing $S_{P} U_{P}^{\top}
  U_{A} = S_{P}
  (U_{P}^{\top} U_{A} -
  W^{*}) + S_{P} W^{*}$ and
  rearranging terms, we obtain
 \begin{equation*}
    \begin{split}
      \|W^{*} S_{A} -
      S_{P} W^{*}\|_{F}\leq & \|W^{*}
      - U_{P}^{\top} U_{A} \|_{F}
      (\|S_{A}\| + \|S_{P}\|)\\
& +
      \|U_{P}^{\top}(A - P)
      R\|_{F}\\
&+  \|U_{P}^{\top}(A -
      P) U_{P}U_{P}^{\top}U_A\|_{F} \\
      \leq & O(1) + O(1)\\
      & + \|U_{P}^{\top}(A -
      P) U_{P}\|_{F}\|U_{P}^{\top}U_A\|
    \end{split}
  \end{equation*}
  asymptotically almost surely. Now, $\|U_{P}^{\top}U_A\| \leq 1$. Hence we can focus on the term  $U_{P}^{\top}(A -
      P) U_{P}$, which is a $d \times d$ matrix
      whose $ij$-th entry is of the form
      \begin{align*}
        u_{i}^{\top} (A - P) u_j &=
        \sum_{k=1}^{n} \sum_{l=1}^{n} (A_{kl} -
        P_{kl}) u_{ik} u_{jl} \\
&= 2 \sum_{k,l : k < l}
        (A_{kl} - P_{kl})u_{ik} u_{jl} - \sum_{k}
        P_{kk} u_{ik} u_{jk}
      \end{align*}
      where $u_{i}$ and $u_j$ are the $i$-th and $j$-th
      columns of $U_{P}$. Thus, conditioned on
      $P$, $u_{i}^{\top} (A - P)
      u_j$ is a sum of independent mean $0$ random
      variables and a term of order $O(1)$. 
        Now, by Hoeffding's inequality, 
 \begin{align*}
   \p&\left[ \bigg|\sum_{k,l : k < l} 2 (A_{kl}- P_{kl}) u_{ik} u_{jl}\bigg|\geq t \right] \\
  &\leq 2\exp
  \Bigl( \frac{-2t^2}{\sum_{k,l : k < l} (2u_{ik} u_{jl})^2}\Bigr) 
\leq 2\exp(-t^2).
 \end{align*}
 Therefore, each entry of
 $U_{P}^{\top}(A - P)
 U_{P}$ is of order
 $O(\log n)$ asymptotically almost surely, and as a consequence,
 $$\|U_{P}^{\top}(A - P)
 U_{P}\|_{F}$$
 is of order
 $O(\log n)$ asymptotically almost surely. Hence, 
$$\|W^{*} S_{A} -
 S_{P} W^{*}\| = O(\log n)$$ asymptotically almost surely. We establish
 $\|W^{*} S_{A}^{1/2} -
 S_{P}^{1/2} W^{*} \|_{F} = O(\log n(n \rho_n)^{-1/2})$ by
 noting that the $ij$-th entry of $W^{*} S_{A}^{1/2} -
 S_{P}^{1/2} W^{*}$ can be written as
 $$W^{*}_{ij} (\lambda_i^{1/2}(A) -
 \lambda_{j}^{1/2}(P)) = W^{*}_{ij} \frac{\lambda_{i}(A) -
 \lambda_{j}(P)}{\lambda_{i}^{1/2}(A) + \lambda_{j}^{1/2}(P)}  $$ 
and that the eigenvalues $\lambda_{i}^{1/2}(A)$ and $\lambda_{j}^{1/2}(P)$ are all
of order $\Omega(\sqrt{n \rho_n})$ (see \cite{sussman12:_univer}). 
\end{proof}
We next present Theorem~\ref{thm:minh_frob}, which extends earlier
results on Frobenius norm accuracy of the adjacency spectral embedding
from \cite{perfect} even when the second moment matrix
$\mathbb{E}[X_1 X_1^{\top}]$ does not have distinct eigenvalues.
\begin{theorem}
  \label{thm:minh_frob}
Let $(A, X) \sim \mathrm{RDPG}(F)$ with sparsity factor $\rho_n$. Let $E_n$ be the event that there exists a rotation matrix $W$ such that
\begin{align*}
  &\|\hat{X} - \rho_n^{1/2} X W
  \|_{F} \\
&= \|(A - P) U_{P}
  S_{P}^{-1/2} \|_{F} + O(\log (n)(n \rho_n)^{-1/2})
\end{align*}
Then $E_n$ occurs asymptotically almost surely.
\end{theorem}
\begin{proof}
Let 
\begin{align*}
R_1 &= U_{P}
  U_{P}^{\top} U_{A} -
  U_{P} W^{*}\\
R_2 &= (W^{*}
      S_{A}^{1/2} - S_{P}^{1/2} W^{*}).
\end{align*}
We deduce that
  \begin{equation*}
    \begin{split}
      \hat{X} - U_{P}
      S_{P}^{1/2} W^{*} =&
      U_{A} S_{A}^{1/2}
      -U_{P} W^{*}
      S_{A}^{1/2}\\
& + U_{P}
      (W^{*}
      S_{A}^{1/2} - S_{P}^{1/2} W^{*}) \\
 =& (U_{A} -
      U_{P} U_{P}^{\top}
      U_{A}) S_{A}^{1/2}\\
& +
      R_1 S_{A}^{1/2} + U_{P}
      R_2 \\ 
=& U_{A} S_{A}^{1/2} -
      U_{P} U_{P}^{\top}
      U_{A} S_{A}^{1/2}\\
&  + R_1 S_{A}^{1/2} + U_{A}
      R_2 
    \end{split}
  \end{equation*}
  Now, $U_{P} U_{P}^{\top}
  P = P$ and $U_{A}
  S_{A}^{1/2} = A U_{A}
  S_{A}^{-1/2}$. Hence
  \begin{align*}
       \hat{X} - U_{P}
      S_{P}^{1/2} W^{*} = &(A -
      P) U_{A}
      S_{A}^{-1/2}\\
& - U_{P}
      U_{P}^{\top} (A - P)
      U_{A} S_{A}^{-1/2} \\
&+ R_1 S_{A}^{1/2} + U_{A}
      R_2 
  \end{align*}
  Writing 
\begin{align*}
R_3 =& U_{A} - U_{P}
  W^{*} \\
=& U_{A} - U_{P}
  U_{P}^{\top} U_{A} + R_1,
\end{align*}
we derive that
  \begin{equation*}
    \begin{split}
    \hat{X} - U_{P}
      S_{P}^{1/2} W^{*} =& (A -
      P) U_{P} W^{*}
      S_{A}^{-1/2}\\
& - U_{P}
      U_{P}^{\top}(A - P)
      U_{P} W^{*}
      S_{A}^{-1/2} \\
 &+ (I -
      U_{P} U_{P}^{\top}) (A - P)
      R_{3} S_{A}^{-1/2}\\
& + R_1 S_{A}^{1/2} + U_{A}
      R_2 
    \end{split}
  \end{equation*}
  Now 
\begin{align*}
\|R_{1}\|_{F} & = O((n \rho_n)^{-1}),\\
 \|R_2\|_{F} &= O(\log n(n \rho_n)^{-1/2}), \textrm{ and } \\
 \|R_3\|_{F} &= O((n \rho_n)^{-1/2});
\end{align*}
indeed, we recall
\begin{align*} 
\|U_{A} - U_{P}
    U_{P}^{\top} U_{A} \|_{F}  = O((n \rho_n)^{-1/2}).
\end{align*}   
Furthermore, Hoeffding's inequality guarantees that
  \begin{align*}
    \|&U_{P}
      U_{P}^{\top}(A - P)
      U_{P}  W^{*}
      S_{A}^{-1/2} \|_{F}  \\
&\leq \|U_{P}^{\top}(A - P)
      U_{P}\|_{F} \|S_{A}^{-1/2}
      \|_{F}= O(\log n(n \rho_n)^{-1/2})
  \end{align*}
 As a consequence, 
  \begin{align*}
     \|&\hat{X} - U_{P}
      S_{P}^{1/2} W^{*}\|_{F}\\
      &=  \|(A -
      P) U_{P} W^{*}
      S_{A}^{-1/2}\|_{F} + O(\log n(n\rho_n)^{-1/2}) \\ 
&= 
      \|(A - P) U_{P}
      S_{P}^{-1/2} W^{*}\\
& \hspace{5mm}+  (A - P) U_{P}
      (S_{P}^{-1/2} W^{*} - W^{*}
      S_{A}^{-1/2}) \|_{F}\\
      &\hspace{5mm}+ O(\log n(n \rho_n)^{-1/2})
  \end{align*}
Using a very similar argument as that employed in the proof of  Lemma~\ref{lem:order_bounds_on_minh_differences}, we can show that
  \begin{equation*}
    \|S_{P}^{-1/2} W^{*} - W^{*}
      S_{A}^{-1/2} \|_{F} = O(\log n(n \rho_n)^{-3/2})
  \end{equation*}
  Recall that 
\begin{align*}
&\|((A - P) U_{P}
      (S_{P}^{-1/2} W^{*} - W^{*}
      S_{A}^{-1/2}) \|_{F}\ \\
&\leq \|(A-P)U_P\|\|(S_{P}^{-1/2} W^{*} - W^{*}
      S_{A}^{-1/2}) \|_{F}
\end{align*}
Further, as already mentioned, Theorem 6 of \cite{lu13:_spect} ensures that $\|(A-P)\|$ is of order $O(\sqrt{n \rho_n})$ asymptotically almost surely; this implies, of course, identical bounds on $\|(A-P)U_P\|)$.  We conclude that 
  \begin{align}
    \label{eq:2}
    \|&\hat{X} - U_{P}
      S_{P}^{1/2} W^{*}\|_{F}\notag\\
      &= \|(A - P) U_{P}
      S_{P}^{-1/2} W^{*} \|_{F}+
      O(\log (n)(n\rho_n)^{-1/2}) \notag\\ &= \|(A - P) U_{P}
      S_{P}^{-1/2} \|_{F} + O(\log (n)(n\rho_n)^{-1/2}).
  \end{align}
  Finally, to complete the proof, we note that 
$$\rho_n^{1/2} X =
  U_{P} S_{P}^{1/2} W$$
  for some orthogonal matrix $W$. Since $W^{*}$ is also
  orthogonal, we conclude that there exists some orthogonal $\tilde{W}$
for which 
$$\rho_n^{1/2} X \tilde{W} = U_{P}
  S_{P}^{1/2} W^{*},$$
  as desired. 
\end{proof}
We are now ready to prove Theorem~\ref{thm:minh_sparsity}.
\begin{proof}
To establish Theorem~\ref{thm:minh_sparsity}, we begin by noting that 
\begin{align*}
\|\hat{X} - \rho_n^{1/2} X W
  \|_{F} &= \|(A - P) U_{P}
  S_{P}^{-1/2} \|_{F}\\
  &\hspace{5mm} + O(\log (n)(n \rho_n)^{-1/2})
  \end{align*}
  and hence
  \begin{align*}
   &\max_{i} \| \hat{X}_i - \rho_n^{1/2} W X_i \|\\ 
   &\leq 
   \frac{1}{\lambda_{d}^{1/2}(P)} 
   \max_{i} \|((A - P) U_{P})_{i} \|+
   O(\log(n)(n\rho_n)^{-1/2})  \\
   & \leq \frac{d^{1/2}}{\lambda_{d}^{1/2}(P)} 
   \max_{j} \|(A - P) u_j \|_{\infty}+ O(\log (n)(n\rho_n)^{-1/2})
   \end{align*}
   where $u_j$ denotes the $j$-th column of
   $U_{P}$. 
   Now, for a given $j$ and a given index $i$, the $i$-th element of
   the vector $(A - P) u_j$ is of the form
   \begin{equation*}
     \sum_{k} (A_{ik} - P_{ik}) u_{jk}
   \end{equation*}
   and once again, by Hoeffding's inequality, the above term is $O(\log n)$ asymptotically almost surely. Taking the union bound
   over all indices $i$ and all columns $j$ of $U_{P}$, we conclude
   \begin{align*}
    \max_{i} \| \hat{X}_i - \rho_n^{1/2} W X_i \|& \leq
    \frac{C d^{1/2}}{\lambda_{d}^{1/2}(P)} \log^2(n)\\
    &\hspace{5mm} + O(\log (n)(n\rho_n)^{-1/2})\\
& \leq \frac{C d^{1/2} \log^2{n}}{\sqrt{n\rho_n}}
   \end{align*}
   as desired.
\end{proof}

\subsection*{Proof of Lemma \ref{lem:perfect}}
Our assumption that $p<q$, and Theorem \ref{thm:minh}, gives that 
$$\hat p:=\max_{i,j:i\neq j}\max_{\ell,h}\langle \widehat\xi^{(2)}_i(\ell,:), \widehat\xi^{(2)}_j(h,:) \rangle,$$
and 
$$\hat q:=\min_{i}\min_{\ell,h}\langle \widehat\xi^{(2)}_i(\ell,:), \widehat\xi^{(2)}_i(h,:) \rangle,$$
are such that $\hat p<\hat q$ asymptotically almost surely.
The proof of Lemma \ref{lem:perfect} follows from the following proposition and the fact that $\hat p<\hat q$ asymptotically almost surely.
\begin{proposition}
\label{prop:casnaaccuracy}
Given the assumptions of Lemma \ref{lem:perfect} and Lemma \ref{thm:minh}, let $E_n$ be the even that the the set $\mathcal{S}_n$ obtained in Algorithm \ref{alg:main0} satisfies 
$$\left|\mathcal S_n\cap \{\widehat\xi^{(2)}_i(\ell,:)\}_{\ell=1}^{|V(H_i)|}\right|=1$$ for all $i\in[R]$.  Then $E_n$ occurs asymptotically almost surely.
\end{proposition}
\begin{proof}
For each $i\in[R]$, define
$C_j=\{\hxi^{(2)}_i(\ell,:)\}_{\ell=1}^{|V(H_i)|}$.
The proposition follows immediately from proving 
\begin{itemize}
\item[(1)] For all $i\in[n]$, if $\hX(i,:)$ belongs to $C_j$ and 
$|\mathcal S_{i-1}\cap C_j|=0$, then $\hX(i,:)$ will be added to $\mathcal S_{i-1}.$
\item[(2)] For all $i\in[n]$, if $s\in \mathcal{S}_{i-1}$ belongs to $C_j$ and $|\mathcal S_{i-1}\cap C_j|=1$, then $s\in \mathcal{S}_{i}$ (i.e., $s$ will not be removed from $\mathcal S_{i-1}$).
\end{itemize}

For (1), for fixed $i\in[n]$, if $\hX(i,:)$ belongs to $C_j$ and $|\mathcal S_{i-1}\cap C_j|=0$, then 
$$\max_{s\in \mathcal S_{i-1}}\langle \hX(i,:),s\rangle\leq \hat p.$$
By the pigeonhole principle, there exist 
$y,z\in \mathcal S_{i-1}$ such that 
$y,z\in C_k$ for some $k\in[R],\, k\neq j$.
Thus $\langle y,z\rangle \geq \hat q$, and 
$$\max_{s\in \mathcal S_{i-1}}\langle \hX(i,:),s\rangle< \max_{x,w\in \mathcal S_{i-1}}\langle x,w\rangle,$$
and hence $\hX(i,:)$ will be added to $\mathcal S_{i-1}.$

For (2), for fixed $i\in[n]$, 
suppose $s\in \mathcal{S}_{i-1}$ belongs to $C_j$ and $|\mathcal S_{i-1}\cap C_j|=1$.
Consider two cases.  First, suppose that for each $k\in[R]$, $|\mathcal S_{i-1}\cap C_k|=1$.  
Then 
$$\max_{s\in \mathcal S_{i-1}}\langle \hX(i,:),s\rangle\geq \hat q>\hat p> \max_{x,w\in \mathcal S_{i-1}}\langle x,w\rangle,$$ 
and $\hX(i,:)$ will not be added to $\mathcal S_{i-1},$ and so $s\in \mathcal S_{i}.$
Otherwise, there exists $y,z\in\mathcal S_{i-1}$ satisfying $y,z\in C_k$ for some $k\in[R]$, $k\neq j$.
Therefore 
$$\max_{x\in\mathcal S_{i-1}}\langle x,s\rangle\leq \hat p<\hat q\leq\langle y,z\rangle\leq \max_{x,w\in\mathcal S_{i-1}}\langle x,w\rangle,$$
and even if $\hX(i,:)$ is added to $\mathcal S_{i-1}$, then $s$ will not be removed from $\mathcal S_{i-1}$, as desired.
\end{proof}
To finish the proof of Lemma \ref{lem:perfect}, from Proposition \ref{prop:casnaaccuracy}, the set $S_n$ will contain a single row of $\hxi^{(j)}$ for each $j\in[R]$ asymptotically almost surely.
For each $i\in[n]$, if $\hX(i,:)\in C_k$, then asymptotically almost surely
$$\text{argmax}_{j}\langle \hX(i,:),s_j\rangle\in C_k,$$
as desired.

\end{document}